\documentclass{article}
\usepackage[letterpaper,margin=1.0in]{geometry}
\usepackage{enumitem}
\usepackage[utf8]{inputenc} 
\usepackage[T1]{fontenc}    
\usepackage{hyperref}       
\usepackage{url}            
\usepackage{booktabs}       
\usepackage{amsfonts}       
\usepackage{nicefrac}       
\usepackage{microtype}      
\usepackage{graphicx}       
\usepackage[numbers,sort]{natbib}
\usepackage{multirow}
\usepackage{bbding}
\graphicspath{{media/}}     

\usepackage{thmtools}
\usepackage{thm-restate}

\usepackage{amsmath}
\usepackage{amsthm}
\usepackage{amssymb}
\usepackage{url}
\usepackage[capitalize,noabbrev]{cleveref}
\crefname{appendix}{Appendix}{Appendices}
\Crefname{appendix}{Appendix}{Appendices}
\crefname{app}{Appendix}{Appendices}
\Crefname{app}{Appendix}{Appendices}
\crefname{subappendix}{Appendix}{Appendices}
\Crefname{subappendix}{Appendix}{Appendices}
\usepackage{enumitem}
\newtheorem{definition}{Definition}
\newtheorem{theorem}{Theorem}
\newtheorem{assumption}{Assumption}

\newtheorem{remark}{Remark}
\usepackage{enumitem}
\usepackage{graphicx}
\usepackage{caption}  
\usepackage{subcaption}  
\usepackage{adjustbox}  

\usepackage{xcolor}
\usepackage[most]{tcolorbox}
\newif\ifcomments
\commentstrue

\newif\ificml
\icmlfalse

\newif\ifneurips
\neuripsfalse

\title{When Does Online Imitation Learning Help in LLM Post-Training? \\
The Role of (Non-)Realizability Beyond Horizon}

\author{
Huaqing Zhang\textsuperscript{*} \\
Tsinghua University \\
\texttt{zhanghq22@mails.tsinghua.edu.cn}
\and
Jingchu Gai\textsuperscript{*} \\
Carnegie Mellon University \\
\texttt{jgai@cs.cmu.edu}
\and
Juno Kim \\
UC Berkeley \\
\texttt{junokim@berkeley.edu}
\and
Bingbin Liu \\
Harvard University \\
\texttt{bliu@g.harvard.edu}
\and
Andrej Risteski \\
Carnegie Mellon University \\
\texttt{aristesk@cs.cmu.edu}
}

\date{}  

\usepackage{mathtools}
\usepackage[most]{tcolorbox}
\usepackage{longtable}
\usepackage{booktabs}  
\usepackage{caption}
\tcbuselibrary{breakable,skins}

\newtcolorbox{responsebox}[2][]{
    breakable,
    enhanced,
    colback=cyan!5,        
    colframe=cyan!50!blue, 
    coltext=black,
    coltitle=white,
    fonttitle=\bfseries\rmfamily,
    arc=1mm,
    boxrule=1pt,
    width=0.95\linewidth,
    center,
    title=#2,
    #1
}

\include{defs}
\def\shownotes{1}  
\ifnum\shownotes=1
\newcommand{\authnote}[2]{{$\ll$\textsf{\footnotesize #1 notes: #2}$\gg$}}
\else
\newcommand{\authnote}[2]{}
\fi

\definecolor{jcg}{RGB}{100,160,0}

\def\eqref#1{equation~\ref{#1}}

\def\1{\bm{1}}

\DeclareMathAlphabet{\mathsfit}{\encodingdefault}{\sfdefault}{m}{sl}
\SetMathAlphabet{\mathsfit}{bold}{\encodingdefault}{\sfdefault}{bx}{n}

\newcommand{\E}{\mathbb{E}}

\newcommand{\R}{\mathbb{R}}

\newcommand{\KL}[2]{D_{\mathrm{KL}}\left({#1}\|{#2}\right)}
\newcommand{\Var}{\mathrm{Var}}

\DeclareMathOperator*{\argmax}{arg\,max}

\def \TV {\text{TV}}

\def \x {x}
\def \y {y}
\def \xdist{\mathcal D_{x}}
\def \xspace{\mathcal X}
\def \yspace{\mathcal Y}
\def \policy {\pi}
\def \policyclass {\Pi}
\def \policyhat {\hat \policy}
\def \expolicy {\policy ^{\mathrm{e}}}

\def \basepolicy {\policy ^{\mathrm{base}}}
\def \r {r}
\def \Jon {J_{\mathrm{on}}}
\def \hatJon {\widehat J_{\mathrm{on}}}
\def \vmax {V_{\mathrm{max}}}
\def \Zi {Z_i^\policy}

\def \val {V}
\def \dTV {D_{\TV}}             
\def \excoverage {C^{\mathrm{e}}_\infty}
\def \basecoveragestar {C^{\mathrm{base}}_\infty}
\def \basecoverage {C^{\mathrm{base}}}
\def \optimalpolicy {\pi^\star}

\def \cbinstance {\mathcal C}

\begin{document}
\maketitle
\begin{abstract}

Online imitation learning (IL), particularly on-policy distillation, has emerged as a strong LLM post-training approach, often outperforming offline supervised fine-tuning (SFT). Yet a principled understanding of when and why online interaction helps remains unclear. 
In this work, we challenge the view that error accumulation is the main source of online IL's advantage, and instead show that the benefits of online interaction depend critically on whether the setting is realizable, i.e., whether the student policy class can represent the expert policy.  
Under realizability, we empirically find that offline IL already matches expert performance.
In contrast, in non-realizable (misspecified) settings, we prove that offline IL encounters an information-theoretic bottleneck even when horizon $H=1$, and propose a structural characterization of misspecification relative to the reward, under which online IL provably achieves high performance despite a large distributional mismatch between the expert and student policies.
\end{abstract}

\begingroup
\renewcommand{\thefootnote}{*}
\footnotetext{Equal contribution.}
\endgroup

\section{Introduction}


Imitation learning (IL) trains a student policy from expert demonstrations, enabling the student to achieve strong performance without access to explicit reward signals.
For post-training of large language models (LLMs), the canonical offline imitation learning method is \textit{supervised fine-tuning} (SFT), where the learner imitates an expert policy using a fixed dataset of expert demonstrations, such as human-written responses or traces generated by a stronger model. However, such offline approaches are reported to suffer from limitations such as poor generalization \citep{chu2025sft, arnav2024falsepromise, kumar2022distortpretrainfeatures} and severe catastrophic forgetting \citep{chen2025retaining}. 

Online imitation learning offers a promising alternative for mitigating these issues in post-training.
In the online IL setting, the learner actively interacts with the environment and queries an expert to obtain feedback \cite{ross2011dagger}. A prominent online IL approach for LLM post-training is \emph{on-policy distillation}, where the student samples from its current policy and minimizes the reverse-KL divergence to the expert distribution \cite{kim2016sequence,agarwal2024policy,gu2023minillm}. Although recent studies report the empirical superiority of online IL in LLM post-training \citep{yang2025qwen3, lu2025onpolicydistillation}, a principled understanding of when and why this advantage arises remains limited.

A classical explanation for the advantage of online IL is its ability to mitigate error accumulation: offline IL can suffer from errors that compound quadratically over the horizon, whereas DAgger-style online algorithms improve this dependence under a recoverability assumption \citep{ross2010efficient,ross2011dagger}.
However, recent theory by \citet{foster2024behaviorcloning} shows that under the standard log-likelihood objective used in SFT, offline IL can match the sample-complexity guarantees of DAgger in \emph{realizable} settings, where the expert policy is contained in the student policy class. This result calls into question whether error accumulation alone can fully explain the practical advantage of online IL in LLM training.

In this work, we identify \textit{non-realizability} (or \textit{misspecification}) as an orthogonal source of online IL's advantage: the student policy class is too restricted, e.g., due to a much smaller model size, to represent the expert policy. Unlike the classical error-accumulation view, which concerns long-horizon compounding, this advantage can already arise in contextual bandits with horizon $H=1$.

We first study the realizable setting as a sanity check, and empirically show that offline IL already suffices in this setting (\Cref{sec:compare_realizable}).
Specifically, we consider both synthetic (Countdown \citep{tinyzero}) and real-world math reasoning (GSM8K \citep{cobbe2021training} and DeepScaleR \citep{luo2025deepscaler}) tasks. 
For each task we first apply reinforcement learning (RL) to a base model to obtain the expert, and then compare SFT and on-policy distillation initialized from the same base model.
By construction, this setup guarantees realizability, since the expert policy is itself obtained within the same model family. 
Our results show that the SFT-trained student matches the expert's performance, and on-policy distillation offers no further gains in either accuracy or training speed.
This finding is consistent with the theory of \citet{foster2024behaviorcloning}, which predicts no statistical-complexity gap between offline and online IL in realizable settings, and further indicates that other potential sources of online IL advantage, such as optimization effects or problem-specific structures, do not yield an empirical advantage in this regime either.
Taken together, these results further motivate our focus on the \textit{non-realizable} setting as the key regime for understanding the benefits of online imitation learning.

In non-realizable settings, we develop an account of online IL's advantage that is orthogonal to the classical error-accumulation argument. 
We focus on a simplified yet informative contextual-bandit framework with horizon $H=1$, where long-horizon compounding is absent by construction, yet online IL still gains an advantage from misspecification.
We first show, both theoretically and empirically, that prior analyses that control the performance gap via distributional discrepancies between the expert and student policies \citep{foster2024behaviorcloning, rohatgi2025computational, rajaraman2020limitil} can be overly pessimistic under misspecification, and reward-dependent analyses are necessary.
For offline IL, we identify an information-theoretic limitation specific to non-realizable settings: the expert policy may place too little probability mass on responses that the restricted student class can learn.
For online IL, we further introduce a structural condition that explicitly captures the misspecification structure relative to the reward (\Cref{ass:main}), 
under which online IL can provably achieve small performance suboptimality despite large distributional discrepancy. At a high level, the condition states that the expert's signal should either be aligned with the actual reward, or deviate from it only in ways that the restricted student class cannot realize.


To summarize, our main contribution is to identify non-realizability as a key factor underlying the advantage of online over offline imitation learning in language model post-training. Specifically, 

\begin{enumerate}[leftmargin=*]
    \item In realizable settings, we empirically verify that offline IL (SFT) fully recovers expert performance, and that online IL (on-policy distillation) yields neither accuracy gains nor faster learning (\cref{sec:compare_realizable}).
    \item In non-realizable settings, we show that prior discrepancy-based analyses can be overly pessimistic (\cref{sec:distance_insufficient}), and reward-dependent analyses are necessary to explain the success of IL.
    \item By analyzing the reward-dependent structure of misspecification, we identify an information-theoretic limitation of offline IL in this regime (\cref{sec:offlinelimitation}) and propose a structural condition characterizing when online IL can remain effective despite misspecification (\cref{sec:onlineefficient}).
\end{enumerate}

\section{Preliminaries}
\label{sec:preliminary}

\ifneurips
\else
For a set $\mathcal X$, let $\Delta(\mathcal X)$ denote the set of all probability distributions over~$\mathcal X$. For two distributions $P,Q$ over~$\mathcal X$, their Kullback--Leibler divergence is defined by $\KL {P}{Q} \coloneqq \mathbb E_{x\sim P}\!\left[\log \frac{P(x)}{Q(x)}\right]$, and their total variation distance is defined by $\dTV(P,Q) \coloneqq \sup_{A\subseteq \mathcal X} |P(A)-Q(A)|=\frac{1}{2}\sum_{x\in\mathcal X}|P(x)-Q(x)|$. We use standard asymptotic notation throughout, and we write $a \lesssim b$ as shorthand for $a = O(b)$, and $a \gtrsim b$ as shorthand for $a = \Omega(b)$.
\fi
\ifneurips
\else
\subsection{Problem Setting}
\fi
To provide a clean theoretical sandbox, we abstract language model post-training as a contextual bandit problem~\citep{rafailov2023dpo, xiong2023iterative, foster2025good}.
Let $\mathcal{X}$ represent the set of all possible prompts (contexts) and $\mathcal{Y}$ the set of generated responses (actions). A policy $\policy: \xspace \to \Delta(\yspace) $ specifies a conditional distribution $\policy(y \mid x)$ over responses~$y$ given a prompt~$x$.
Let  $\r(\x,\y):\xspace \times \yspace \to [0,1]$ be a reward function. We aim to find a policy~$\pi$ within a policy class~$\policyclass$ that achieves high expected reward
$\val(\policy )\coloneqq \E_{\x\sim\xdist} \E_{\y\sim \policy(\cdot \mid \x )} \r(\x,\y)$, where $\xdist$ is the distribution over prompts. 
We follow standard convention in RL theory and assume the policy class $|\policyclass|<\infty$ is finite \citep{foster2023foundations, foster2024behaviorcloning, rohatgi2025computational}. 
\ifneurips
\else
This finite-class abstraction provides a clean testbed for studying statistical generalization without imposing stronger structural assumptions, such as linearity, that need not hold for neural networks. \citep{foster2024behaviorcloning}. 
\fi
\ifneurips
\else
\paragraph{Realizability.} 
\fi
\ifneurips
\else

Realizability is a widely studied assumption in RL theory that asks whether the student policy class is expressive enough to represent the expert policy. 
In this work, we identify realizability as a crucial factor governing the advantage (or lack thereof) of online IL over offline IL.  
\fi



\ifneurips
\else
\begin{definition}[Realizability and Misspecification]
We say that the expert policy $\expolicy$ is \textit{realizable} with respect to the student policy class $\Pi$ if $\expolicy \in \policyclass$. Otherwise if $\expolicy \notin \policyclass$, the setting is \textit{non-realizable} (or \textit{misspecified}).
\end{definition}
\fi

\ifneurips
\else
\subsection{Imitation Learning}
\label{sec:pre-il}
\fi

Imitation learning is a learning paradigm in which the learner improves its policy by leveraging information provided by an expert policy $\expolicy$, without direct access to the reward feedback during training.
In this work, we study two settings: offline and online imitation learning.

\ifneurips
\else
\paragraph{Offline Imitation Learning.}
The learner is given a static dataset $\mathcal{D}=\{(x_i,y_i)\}_{i=1}^n$ of i.i.d.\ expert demonstrations, where $x_i\sim\xdist$ and $y_i\sim\expolicy(\cdot\mid x_i)$.
The standard approach in this setting is \textbf{behavior cloning} (BC), which converts IL to a supervised learning problem with the following objective:
\begin{definition}[Population Objective of Behavior Cloning]
    Given an expert policy $\expolicy$ and a loss function $\ell: \Delta(\mathcal{Y})\times\mathcal{Y}\to\mathbb{R}$, the population behavior cloning objective is defined as
    \begin{align}
    \vspace{-6mm}
    \label{eq:offlineil}
   L_{\mathrm{BC}}(\pi)\coloneqq \E_{x\sim\xdist}\,\E_{y\sim\expolicy(\cdot\mid x)}\big[\ell(\pi(\cdot\mid x),y)\big].
    \vspace{-6mm}
    \end{align}
    The goal of behavior cloning is to find a policy $\pi \in \Pi$ that minimizes $L_{\mathrm{BC}}(\pi)$.
\end{definition}
A standard instantiation of behavior cloning in LLM post-training uses the negative log-likelihood loss $\ell(\pi(\cdot\mid x),y) = -\log \pi(y\mid x)$,
which corresponds to the \textbf{supervised fine-tuning (SFT)} objective. 
\fi

\ifneurips
\else
\paragraph{Online Imitation Learning.} In the online setting, the learner can actively query the expert for feedback.
Learning proceeds in episodes:
in the $i^{\text{th}}$ episode, the learner observes $\x_i \sim \xdist$, executes a certain policy $\policy_i$ to generate a response $\y_i\sim \policy_i(\cdot \mid \x_i)$, and receives expert-provided information evaluated on this response.
In existing online imitation learning formulations based on MDPs (e.g., \citep{ross2010efficient, foster2024behaviorcloning}), the expert provides step-wise guidance along a trajectory.
In contrast, we adopt a simpler contextual bandit formulation: in each episode, the learner produces a complete response $\y_i$ (a token sequence), and the expert reveals its conditional probability $\expolicy(\y_i\mid \x_i)$ on that response. 
This form of probability access is natural in language model distillation, where the expert model’s logits are available.
Furthermore, it captures the core interaction pattern of widely used on-policy distillation algorithms \citep{agarwal2024policy, gu2023minillm, lu2025onpolicydistillation}, and enables a meaningful characterization of the benefit of online access to expert information beyond a fixed offline dataset (\Cref{sec:compare_non_realizable}).
\fi

\ifneurips
\else
While the objective of online IL can be flexible, in this work we focus on the following formulation:
\begin{definition}[Population Objective of Online Imitation Learning]
\label{def:onlineil}
Given an expert policy $\expolicy$ and a \emph{shaping function} $f:\mathbb{R}\rightarrow\mathbb{R}$ that maps the expert density to a scalar score, the population objective of online imitation learning is
\begin{equation}
\label{eq:onlineil}
J_{\mathrm{on}}(\pi)
\coloneqq
\mathbb{E}_{x\sim\xdist}
\mathbb{E}_{y\sim\pi(\cdot\mid x)}
\!\left[
f\!\left(\expolicy(y\mid x)\right)
\right].
\end{equation}
The goal of online imitation learning is to find a policy $\pi \in \Pi$ that maximizes $J_{\mathrm{on}}(\pi)$.
\end{definition}
The key difference to the offline objective (\Cref{eq:offlineil}) is that the inner expectation is taken over the student policy $\pi$, rather than the expert policy $\expolicy$. 
\fi
\ifneurips
\else

A prominent instantiation of online imitation learning in language model post-training is \textbf{on-policy distillation} \citep{agarwal2024policy, gu2023minillm}, which uses a reverse-KL objective and is optimized in an online, on-policy manner:
\begin{equation}
\label{eq:reverse-kl}
\max_{\pi\in\policyclass}
-\mathbb{E}_{\x\sim\xdist}
\!\left[
\KL{\pi(\cdot\mid \x)}{\expolicy(\cdot\mid \x)}
\right].
\end{equation}
This reverse-KL objective is an instance of \Cref{def:onlineil} with $f\left(\expolicy(\y\mid \x)\right)=\log \left(\expolicy(\y\mid \x)\right)$, together with an additional entropy regularization term that promotes exploration and prevents policy collapse:\footnote{The shaping function $f$ can be generalized to depend on both $\expolicy(\y\mid \x)$ and $\pi(\y\mid \x)$ without affecting our analysis. This extension fully captures the reverse-KL objective. For simplicity, we focus on shaping functions that depend only on the expert density.}
\ificml
\begin{align*}
\label{eq:reverse-kl-decomp}
&-\mathbb{E}_{\x\sim\xdist}
\!\left[
\KL{\pi(\cdot\mid \x)}{\expolicy(\cdot\mid \x)}
\right]\\
=&
\;\mathbb{E}_{x\sim\xdist}\mathbb{E}_{y\sim\pi(\cdot\mid \x)}
\!\left[\log \expolicy(y\mid \x)\right] 
+
\mathbb{E}_{\x\sim\xdist}
\!\left[
\mathcal{H}\!\left(\pi(\cdot\mid \x)\right)
\right],
\nonumber
\end{align*}
\else 
\begin{align*}
\label{eq:reverse-kl-decomp}
&-\mathbb{E}_{\x\sim\xdist}
\!\left[
\KL{\pi(\cdot\mid \x)}{\expolicy(\cdot\mid \x)}
\right]
=
\mathbb{E}_{\x\sim\xdist}\mathbb{E}_{y\sim\pi(\cdot\mid \x)}
\!\left[\log \expolicy(\y\mid \x)\right] 
+
\mathbb{E}_{\x\sim\xdist}
\!\left[
\mathcal{H}\!\left(\pi(\cdot\mid \x)\right)
\right],
\nonumber
\end{align*}
\fi
where $\mathcal{H}(\pi(\cdot\mid \x)) \coloneqq - \mathbb{E}_{\y\sim\pi(\cdot\mid \x)}[\log \pi(\y\mid \x)]$ denotes the policy entropy.
\fi

\ifneurips
\else

\begin{remark}[Contextual bandit formulation]
Previous analyses of imitation learning are typically framed in the Markov Decision Process (MDP) setting, which involves sequential decision-making, stochastic state transitions, and rewards assigned at each step. In this work, we focus on autoregressive language generation, which is a simpler special case: the state is accumulative with deterministic transition $s_t = (x, y_{\leq t}) = (s_{t-1}, y_t)$, and the reward (e.g., whether a math problem is solved correctly) is defined over the entire trajectory. As a result, this setup is essentially a contextual bandit: given a context $x$, the policy $\pi$ specifies the distribution over the entire response trajectories $y$.
This formulation also matches the practical objectives used in LLM post-training. In particular, both SFT (\cref{eq:offlineil}) and on-policy distillation (\cref{eq:reverse-kl}) are defined at the full-trajectory level, and can therefore be naturally captured within the contextual bandit framework.


Moreover, in this work, we identify \emph{non-realizability} as a source of advantage of online over offline imitation learning that is distinct from the classical error-accumulation view. Unlike the latter, which attributes the benefit of online interaction to the mitigation of compounding errors over long horizons \citep{ross2010efficient}, the advantage arising from non-realizability already appears in contextual bandits with $H=1$. By removing the multi-step structure, the contextual bandit formulation provides a cleaner setting for isolating this source of advantage.
\end{remark}
\fi





\section{Related Work}
\label{sec:related}

\ifneurips
\else
\paragraph{Imitation Learning for LLM Post-Training.}
\fi
Imitation learning plays an important role in language-model post-training. Supervised fine-tuning (SFT) can be viewed as offline IL, where the model is trained with a log-likelihood objective on the demonstration data \citep{brown2020lfewshotlearner, radford2021learning, ouyang2022training}. 
\ifneurips
\else
However, empirical studies show that SFT can struggle to improve a weak student when the expert is a substantially stronger reasoning model. In these settings, expert responses may be too compressed or rely on reasoning patterns beyond the student's representational capacity, making them difficult for the student to imitate effectively
\citep{li2025small, jiang-etal-2025-teach}.
\fi

\ifneurips
\else

Recently, \textbf{online imitation learning} has emerged as a powerful paradigm for LLM post-training. A representative instantiation in this context is on-policy distillation, where the learner actively samples responses from the current student model and queries the expert for probabilities along these self-generated trajectories \cite{kim2016sequence, gu2023minillm}. Unlike supervised fine-tuning that optimizes a forward-KL objective on fixed data, online distillation typically minimizes the reverse-KL divergence between the student and expert distributions \cite{gu2023minillm}. 
Recent work has further extended on-policy distillation to black-box settings without access to expert logits \citep{blackboxopd}, multi-expert distillation \citep{xiao2026mimo}, and self-distillation \citep{selfOPD}.
These methods have proven effective in practice \citep{lu2025onpolicydistillation,yang2026learning} and have been adopted in industrial-scale post-training pipelines \citep{yang2025qwen3,xiao2026mimo}, while more recent work studies their mechanisms and failure modes
\citep{li2026rethinkingopd}.
However, a theoretical understanding of when and why online methods outperform offline SFT remains limited, as discussed below.
\fi

\ifneurips
\else
\paragraph{Theoretical Analysis on Online and Offline IL.}
\fi
\ifneurips
\else
Seminal work \citet{ross2010efficient} shows that offline IL incurs error accumulation:
If the learner incurs a per-step error $\epsilon$ under the expert distribution, then when rolled out under the learner's own policy, the compounding error can scale quadratically with the horizon, as $O(H^2\epsilon)$. 
Online IL algorithms such as DAgger \citep{ross2011dagger, ross2014reinforcement} improve this dependence to $O(\mu H\epsilon)$, where a small recoverability coefficient $\mu$ indicates that the expert can recover effectively from suboptimal intermediate actions. 
However, this line of work gives only algorithm-dependent lower bounds, rather than an end-to-end separation between offline and online IL \citep{foster2024behaviorcloning,rajaraman2020limitil}.

Indeed, \citet{foster2024behaviorcloning} show that, under \textit{realizability}, DAgger-like algorithms offer no worst-case statistical improvement over offline IL with log-likelihood loss.
While online interaction could theoretically yield improvements under additional assumptions (see the discussion in their Section 4), our empirical results do not provide evidence of such gains under realizability (\Cref{sec:compare_realizable}). 
Prior work has also studied the minimax sample complexity of imitation learning
in tabular MDPs and policy classes
\citep{rajaraman2020limitil,Rajaraman2021value,swamy2022minimax}.
\fi


\ifneurips
\else
In non-realizable settings, \citet{rohatgi2025computational} formulate IL as minimizing the Hellinger distance between the student and the expert. They show that a specific offline behavior cloning algorithm, \texttt{RhoEstimatorBC}, already achieves optimal statistical complexity for this objective \citep[Theorem 3.1]{rohatgi2025computational}, and no computational advantage of online interaction is established. This suggests that discrepancy-based analyses can be overly pessimistic and may not explain the practical advantage of online IL.
\fi


\ifneurips
\else
Another line of work is inverse reinforcement learning (IRL), which aims to recover an underlying reward function from expert demonstrations \citep{abbeel2004apprenticeship, ziebart2008maximum, syed2007game, chang2021mitigating}. Recent work shows that interaction with the environment can enable efficient imitation learning under misspecification through IRL, given certain structural assumptions \citep{espinosa2025efficient}. However, these results typically assume that the reward function lies in a finite and realizable hypothesis class, and the IRL algorithms are not commonly adopted in LLM post-training. Finally, a recent study \citep{song2025distill} also empirically identifies misspecification, rather than sampling error, as a more fundamental cause of behavior cloning’s suboptimal performance in partially observed MDPs.
\fi

\section{Offline IL Suffices Under Realizability}
\label{sec:compare_realizable}
Classical analyses attribute the advantage of online IL to mitigating long-horizon error accumulation \citep{ross2010efficient, ross2011dagger}. 
However, recent theoretical work \citep{foster2024behaviorcloning} shows that, for the log-likelihood objective used in SFT, offline behavior cloning already achieves optimal statistical complexity in realizable settings, leaving no worst-case statistical advantage for online IL.
That said, this theory does not cover every possible advantage of online IL in realizable settings. For example, it studies statistical complexity but does not consider optimization dynamics, and it does not rule out the possibility that online IL may still benefit from favorable problem-specific structure.
In this section, as a sanity check for the realizable setting, we empirically demonstrate that SFT (offline IL) is indeed sufficient to fully recover expert performance in this regime, and on-policy distillation does not yield further gains.  

\subsection{Experimental Setup}
    
    

\ifneurips
\else
\begin{figure}
    \centering
    \includegraphics[width=\linewidth]{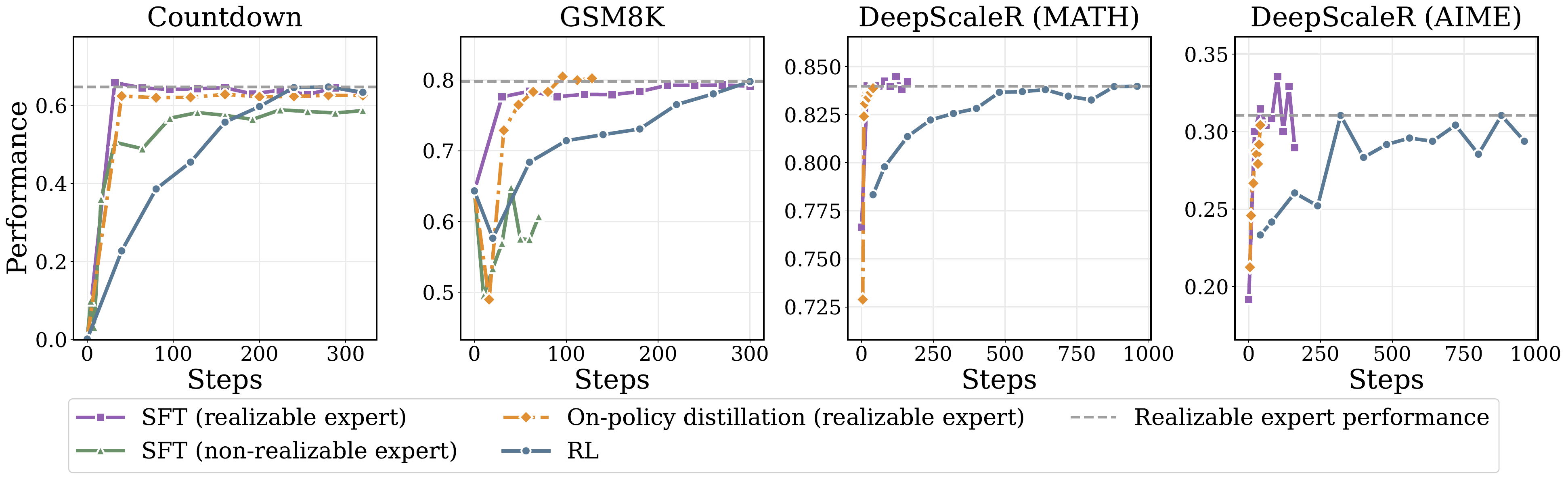}
    \caption{\textbf{Offline IL suffices with realizable experts.}
Across different tasks, SFT with a realizable expert fully matches expert performance, and on-policy distillation yields no accuracy gains or faster training. In contrast, when the expert is non-realizable, SFT exhibits a significant performance gap. }
    \label{fig:in-dist}
    \vspace{-4mm}
\end{figure}
\fi
We conduct experiments on three tasks: a synthetic Countdown \citep{tinyzero} task\footnote{Countdown asks the model to use each given integer exactly once, with operators $\{+,-,\times,\div\}$, to form an expression matching a target; e.g., for $[8,3,2,1]$ and target $24$, $8\times3\times(2-1)=24$.} and two math reasoning datasets: GSM8K~\citep{cobbe2021training} and DeepScaleR \citep{luo2025deepscaler} (evaluated on MATH \citep{hendrycks2021measuring} and AIME benchmarks). 
We use Qwen2.5-3B-Instruct \citep{team2024qwen2} as the base model for Countdown, 
Llama-3.2-3B \citep{grattafiori2024llama} fine-tuned on OpenR1 \citep{face2025open} for GSM8K, 
and DeepSeek-R1-Distill-Qwen-1.5B \citep{deepseekai2025deepseekr1incentivizingreasoningcapability} for DeepScaleR (details in \cref{app:experiment_details}).
We compare four training paradigms:

\begin{itemize}
\item \textbf{RL (GRPO):} We train the model using GRPO~\citep{shao2024deepseekmath0} with verifiable outcome rewards to obtain the realizable expert model for SFT and on-policy distillation.

\item \textbf{SFT with realizable expert:} The student model is trained by SFT on samples generated by the RL-trained model. 
By construction, this setup guarantees realizability, since the expert is obtained within the same model family.

\item \textbf{SFT with non-realizable expert:}
For the Countdown task, the student is trained directly on ground-truth solutions.
For GSM8K, the student is supervised by a highly capable non-realizable expert, DeepSeek-V3.2-Exp~\citep{liu2024deepseek}, which achieves over $95\%$ accuracy on GSM8K.

\item \textbf{On-policy distillation with realizable expert:} The student model is trained by on-policy distillation using the same RL-trained expert described above.
\end{itemize}

\subsection{Experimental Results}

As shown in \Cref{fig:in-dist}, on both the synthetic Countdown task and real-world math reasoning tasks, offline IL (SFT) with a realizable expert matches expert performance, while online IL (on-policy distillation) improves neither final performance nor training efficiency.
This provides no evidence of the classical error-accumulation phenomenon of offline IL in the realizable setting.
Moreover, SFT is substantially more sample-efficient than RL training. 
In contrast, when the expert is non-realizable, SFT falls well short of the expert’s performance. 
This indicates that misspecification is a key ingredient in online IL outperforming offline IL in practice.

We further evaluate out-of-distribution (OOD) generalization on the Countdown task and catastrophic forgetting during GSM8K training. We find that both SFT and on-policy distillation with a realizable expert achieve OOD performance close to that of the expert and exhibit minor forgetting. By contrast, SFT with a non-realizable expert exhibits substantially worse OOD generalization and more severe forgetting. Detailed settings and results are provided in \cref{app:additional_experiments}.

\section{Non-Realizability as a Source of Online IL Gains}
\label{sec:compare_non_realizable}
Since the advantage of online IL in LLM post-training does not arise under realizability (\Cref{sec:compare_realizable}), we turn to the \emph{non-realizable} setting, where the expert policy cannot be represented by the student policy class, i.e. $\expolicy\not\in\policyclass$.
In this setting, online on-policy distillation has been empirically shown to significantly outperform offline IL (SFT)
\citep{gu2023minillm,yang2025qwen3, yang2026learning}\footnote{\citet{gu2023minillm,yang2025qwen3}, and \citet{yang2026learning} study strong-to-weak distillation (distilling a smaller model from a much larger expert), and their setup therefore falls into the non-realizable setting.}.
However, despite these successes, a theoretical understanding of \emph{why and when} online interaction helps in non-realizable settings remains limited.


Previous work framed imitation learning as minimizing the distributional discrepancy (e.g., Hellinger distance) between the student and expert policies \citep{rohatgi2025computational}. 
In this section, we first discuss and empirically verify that such discrepancy-based analyses can be overly pessimistic under non-realizability. 
Furthermore, discrepancy analysis can be misaligned with the true reward objective, and therefore does not adequately explain the widely observed empirical advantage of online over offline IL methods.

Motivated by these limitations, we develop a new perspective on the advantage of online over offline IL in non-realizable settings by directly accounting for the \textit{reward structure}.
For offline IL, we identify a fundamental information-theoretic barrier under non-realizability even in contextual bandits -- 
a barrier distinct from classical error-accumulation arguments \citep{ross2010efficient, ross2011dagger}.
We show in \Cref{thm:lowerboundoffline} that the sample complexity of any offline IL algorithm that learns from i.i.d.\ expert demonstrations scales with a coverage coefficient $\excoverage$ (\Cref{eq:coverage}), which can be interpreted as a measure of the mismatch between the expert and student policies.
For online IL, we characterize the interplay between 
(i) the expert signal, (ii) the misspecified student policy class, and 
(iii) the reward function, and identify a general condition under which online IL is provably efficient despite severe misspecification (\cref{thm:main}).

\vspace{-1mm}
\subsection{Limitations of Discrepancy-based Analyses in Non-Realizable Settings}
\label{sec:distance_insufficient}

\ifneurips
\else
A common approach to deriving IL guarantees is to upper bound the performance suboptimality by a distributional discrepancy between the expert policy $\expolicy$ and the learned policy $\policyhat \in \policyclass$ \citep{rajaraman2020limitil,foster2024behaviorcloning,rohatgi2025computational}. For instance, one typically argues that
\begin{equation}
\label{eq:tv_bound_value}
\textstyle \val(\expolicy) - \val(\policyhat)
\;\le\;
\E_{\x \sim \xdist}\!\left[\dTV\!\left(\expolicy(\cdot \mid \x), \policyhat(\cdot \mid \x)\right)\right],
\end{equation}
where $\dTV$ can be further controlled by other statistical distances or divergences, such as KL, Hellinger, or $\chi^2$ divergences.\footnote{\citet{foster2024behaviorcloning} gives a tighter upper bound based on Hellinger distance than the direct TV bound for deterministic experts or experts with bounded reward variance (Theorems 2.1 and 3.1). Our discussion is not sensitive to this distinction.}
\fi


In realizable settings, the approximation term vanishes:
$\inf_{\policy \in \policyclass}\;
\E_{\x \sim \xdist}\!\left[\dTV\!\left(\expolicy(\cdot \mid \x), \policy(\cdot \mid \x)\right)\right]
= 0,$
so minimizing discrepancy drives the performance suboptimality $\val(\expolicy) - \val(\policyhat)$ to zero. 
In non-realizable settings, however, even the best-in-class discrepancy is strictly positive, so discrepancy-based bounds cannot certify suboptimality below this irreducible distance. 
While such bounds are tight in the worst case over all possible reward functions, they can be overly pessimistic for ``average-case'' reward functions: a large expert-student discrepancy does not necessarily imply a large performance gap.
This is especially relevant for language models, where discrepancies may arise from wording, style, or reasoning traces that are not related to correctness.
For example, a strong expert may solve a math problem directly, while a weaker student may reach the same answer through trial and error; their response distributions can differ substantially despite similar success rates. We verify this empirically  in LLM post-training below.

Furthermore, discrepancy-based analyses do not explain the practical gains of online IL over offline IL. 
Prior work has shown that offline $\rho$-estimator behavior cloning is already statistically optimal for minimizing the discrepancy even in non-realizable settings \citep[Theorem 3.1]{rohatgi2025computational}; see \cref{sec:related} for more details. This suggests that the advantage of online interaction, when it arises, is not in identifying the policy closest to the 
\ifneurips
\else
expert under a distributional discrepancy,
\fi but in identifying the student policy with higher reward under a particular reward function.

\ifneurips
\else

\paragraph{Empirical Validation.}
We empirically validate that the distributional discrepancy between an expert and a student learned via imitation learning can be substantially larger than their performance gap. Specifically, we take DeepSeek-R1-0528, a 692B-parameter model, as the expert, and DeepSeek-R1-0528-Qwen3-8B as the student 
\citep{deepseekai2025deepseekr1incentivizingreasoningcapability}, which is distilled from DeepSeek-R1-0528 using Qwen3-8B-Base \citep{yang2025qwen3} as the base model. We evaluate both models on the AIME 2024 and AIME 2025 benchmarks. For each benchmark, each model generates 14 responses per question, yielding 840 responses per model in total.

As shown in \Cref{fig:length-keywords}, the two models exhibit substantial differences in their response distributions, as reflected by two simple and interpretable features: response length (discretized into 5{,}000-token bins, and measured with the Qwen3-8B tokenizer for both models) and the frequency of keywords associated with reasoning behavior in the responses \cite{zeng2025simplerl, gandhi2025cognitive}.
These metrics are lower bounds of the total variation distance between the response distributions, since distinct responses may share the same length bin or selected keywords.
Thus, the true expert-student discrepancy can only be larger than reported.
Since even these lower bounds already exceed the observed performance gap by a wide margin,
our results provide empirical validation that distributional discrepancies can be much larger than performance differences.
\fi

\vspace{5mm}
These limitations motivate us to go beyond the reward-agnostic discrepancy-based analysis in non-realizable settings, and instead ask when offline IL becomes inefficient (\cref{sec:offlinelimitation}) and when online IL can yield improvements (\cref{sec:onlineefficient}) under a specific reward function.


\ificml
\else
\begin{figure}[t]
    \centering
    \begin{minipage}[t]{0.46\linewidth}
        \vspace{-2pt}
        \centering
        \includegraphics[width=\linewidth]{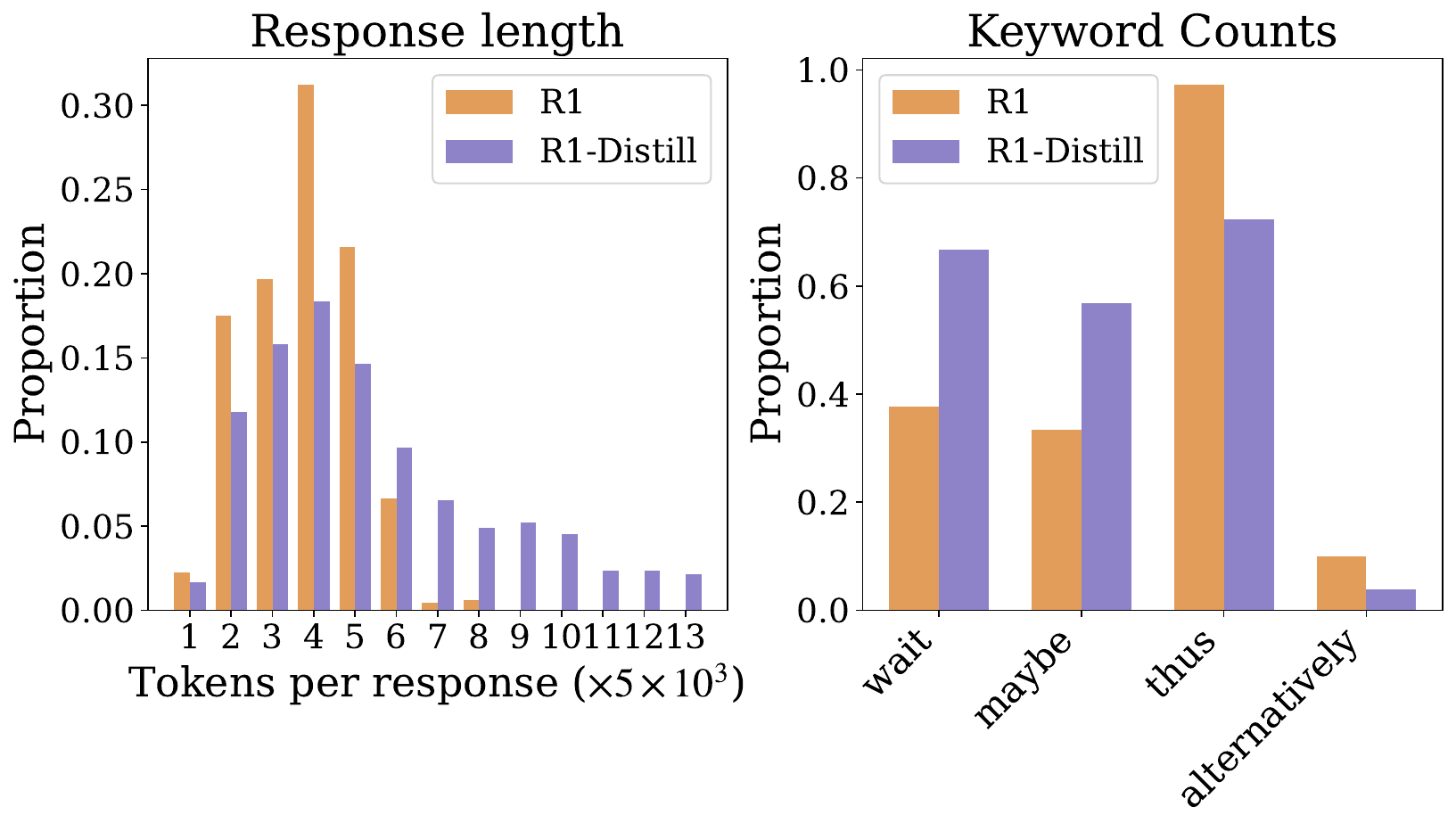}
        \vspace{-14pt}
        
        \textbf{(a)} Distributional discrepancy proxies.
    \end{minipage}
    \hfill
    \begin{minipage}[t]{0.50\linewidth}
        \vspace{3pt}
        \centering
        {\small
        \begin{tabular}{l c}
        \toprule
        \textbf{Metric} & \textbf{Value (\%)} \\
        \midrule
        Length-induced TV distance & 29.99 \\
        Keyword-frequency gap (\textit{wait}) & 29.12 \\
        Keyword-frequency gap (\textit{maybe}) & 23.29 \\
        Keyword-frequency gap (\textit{thus}) & 24.88 \\
        Keyword-frequency gap (\textit{alternatively}) & 6.20 \\
        \midrule
        Performance gap & 5.46 \\
        \bottomrule
        \end{tabular}
        }
        
        \ifneurips
        \vspace{9pt}
        \else
        \vspace{28pt}
        \fi
        \textbf{(b)} Discrepancy vs. performance gap.
    \end{minipage}
    \vspace{-4pt}
    \ifneurips
    \caption{\small\emph{
    \textbf{Bounding performance gap by expert-student discrepancy can be overly pessimistic.}
    We compare DeepSeek-R1-0528 (expert) and DeepSeek-R1-0528-Qwen3-8B (student) on AIME benchmark.
    \textbf{(a)} The expert and student differ substantially in response-length distributions and keyword frequencies.
    \textbf{(b)} Length-induced TV distance and keyword-frequency gaps are lower bounds of the TV distance between response distributions, yet already far exceed the performance gap.
    This suggests that discrepancy-based guarantees can be overly loose.
    }}
    \else
    \caption{{
    \textbf{Bounding performance gap by expert-student discrepancy can be overly pessimistic.}
    We compare DeepSeek-R1-0528 (expert) and DeepSeek-R1-0528-Qwen3-8B (student) on AIME benchmark.
    \textbf{(a)} The expert and student differ substantially in response-length distributions and keyword frequencies.
    \textbf{(b)} Length-induced TV distance and keyword-frequency gaps are lower bounds of the TV distance between response distributions, yet already far exceed the performance gap.
    This suggests that discrepancy-based guarantees can be overly loose for explaining reward performance.
    }}
    \fi
    \vspace{-5mm}
    \label{fig:length-keywords}
\end{figure}
\fi

\subsection{Information-theoretic Lower Bounds for Offline IL in Non-Realizable Settings}
\label{sec:offlinelimitation}

We now present an information-theoretic limitation of offline IL in non-realizable settings. Intuitively, the hard case arises when the expert's responses cannot be faithfully represented by the student policy class, while correct responses that are representable by the student are rarely revealed in expert samples.

We formalize this intuition in Theorem~\ref{thm:lowerboundoffline}, showing that even when the student class $\policyclass$ contains an optimal policy and the expert $\expolicy$ itself achieves perfect performance, learning from i.i.d.\ expert demonstrations can still be statistically hard in non-realizable settings.
Specifically, for any offline IL algorithm that observes only i.i.d.\ expert demonstrations, the sample complexity required to achieve a vanishing performance suboptimality must scale with the coverage coefficient $\excoverage$ defined as
\begin{align}
\label{eq:coverage}
\excoverage \;\coloneqq\; \sup_{\x\in\mathcal{X},\y\in\mathcal{Y}}\frac{\optimalpolicy(\y \mid \x)}{\expolicy(\y \mid \x)},
\end{align}
where $\optimalpolicy \in \arg\max_{\policy \in \policyclass} \val(\policy)$ \citep{munos2003error,chen2019information,rashidinejad2021bridging}.
This barrier is \textbf{specific to non-realizability}: $\excoverage=1$ under realizability since one can take $\pi^\star = \pi^e$, but can be arbitrarily large under misspecification, serving as a measure of the mismatch between the expert and $\pi^\star$.

\label{sec:lowerboundoffline}
\begin{tcolorbox}[
  enhanced, breakable,
  colframe=blue!40!black,
  boxrule=0.35pt, arc=1mm,
  title={
    \parbox{7.5cm}{
      \centering \textbf{Limitation of offline IL in non-realizable settings}
    }
  },
  coltitle=black, fonttitle=\small\sffamily\bfseries,
  colbacktitle=blue!15!white,
  colback=blue!5!white,
  boxed title style={
    sharp corners, boxrule=0pt,
    top=1pt, bottom=0.5pt, left=4mm, right=4mm,
    borderline={0.5pt}{0pt}{blue!20!white}
  },
  attach boxed title to top left={xshift=4mm,yshift*=-1.2mm},
  boxsep=1.5mm, top=1.5mm, bottom=1.5mm, left=2.5mm, right=4mm,
  before skip=10pt, after skip=10pt
]
\begin{theorem}
\label{thm:lowerboundoffline}
For any integer $K \ge 2$, any coverage parameter $\excoverage \ge 1$, and any (possibly randomized) offline imitation learner $\mathcal A$ with access to $N$ i.i.d.\ samples $(x_i,y_i)_{i=1}^N$ with
$x_i \sim \xdist$ and $y_i \sim \expolicy(\cdot \mid x_i)$, there exist
a contextual bandit problem $(\xdist, r)$,
an expert policy $\expolicy$,
and a student policy class $\policyclass$ such that the following hold:
\begin{enumerate}[topsep=5pt,leftmargin=15pt]\setlength{\itemsep}{0pt}
    \item The expert is optimal, i.e., $\val(\expolicy) = 1$.
    \item There exists an optimal student policy $\optimalpolicy \in \policyclass$ with $\val(\optimalpolicy) = 1$.
    \item The expert satisfies a pointwise coverage condition with respect to $\optimalpolicy$, i.e., $\excoverage <\infty$.
\end{enumerate}
Moreover, the policy $\policyhat$ output by $\mathcal A$ satisfies
\[
\textstyle 1-\mathbb E\!\left[ \val(\policyhat)\right]
    \;\gtrsim\;
    \min\!\big\{1,\ \frac{\excoverage \log |\policyclass|}{N}\big\},
\]
where the expectation is over the randomness of the learner and sampling process. 
\end{theorem}
\end{tcolorbox}
The proof constructs a hard instance where, at each context, the expert puts most mass on an optimal-reward action that is unavailable to the student class due to misspecification, while the unique optimal student action appears in expert samples only with probability $1/\excoverage$.
The argument follows Theorem~6.2 of \citet{rajaraman2020limitil}; see \Cref{app:proof-lowerbound}.




\subsection{A Structural Condition for Effective Online IL in Non-Realizable Settings}
\label{sec:onlineefficient}

Online interaction can alleviate the information bottleneck of learning from i.i.d.\ expert demonstrations by querying the expert on learner-induced inputs.
However, these queries may fall in regions rarely sampled by the expert, where the expert signal can be unreliable.
Thus, a high-performing expert is not necessarily a useful teacher for a misspecified student; its feedback must help the learner identify high-performing policies within the restricted student class $\policyclass$.
Motivated by \Cref{sec:distance_insufficient}, we seek structural conditions characterizing the interplay of misspecification structure and reward structure, under which online IL achieves small performance suboptimality $\max_{\policy \in \policyclass} \val(\policy) - \val(\policyhat) $, 
even when the learned policy remains far from the expert in distributional discrepancy. 
We consider a setting in which the feedback on a response $\y$ is given by the expert’s conditional probability $\expolicy(\y\mid\x)$, and study a specific class of online imitation learning objectives (\cref{def:onlineil}):
\begin{equation*}
\textstyle
\max_{\policy \in \policyclass} J_{\mathrm{on}}(\policy)
\coloneqq
\mathbb{E}_{x\sim\xdist}
\mathbb{E}_{y\sim\policy(\cdot\mid x)}
\!\left[
f\!\left(\expolicy(y\mid x)\right)
\right],
\end{equation*}
where $f:\mathbb{R}\rightarrow\mathbb{R}$ is a shaping function.
\ifneurips
\else
As an example, the widely used on-policy distillation algorithm for language model post-training uses $f(\expolicy(\y\mid\x)) = \log (\expolicy(\y\mid\x))$ together with an additional entropy regularization term (\cref{sec:pre-il}).
\fi


\ifneurips
\else
\paragraph{A Motivating Synthetic Example.} 
\fi
We use a simple toy 2D Gaussian example (\Cref{fig:synthetic}) to motivate our condition. 
Actions are $a=(x,y)\in\mathbb R^2$, policies are Gaussians $\pi_\mu=\mathcal N(\mu,I_2)$, and the reward is $r(x,y)=\mathbf 1[y\ge 0]$.
The expert is $\pi^{\mathrm e}=\mathcal N((-2,2),I_2)$, and we compare two student classes $\policyclass_A=\{\mathcal N(\mu,I_2):\mu_x=0\}$ and $\policyclass_B=\{\mathcal N(\mu,I_2):\mu_y=\mu_x\}.$
Both classes are misspecified, yet the expected rewards resulting from IL (with reverse-KL objective as in on-policy distillation) differ sharply.
Let $\hat\policy_A,\hat\policy_B$ be the reverse-KL projections of $\pi^{\mathrm{e}}$ 
onto $\policyclass_A$ and $\policyclass_B$, respectively. Then $\val(\hat\policy_A)\approx 0.977,
\val(\hat\policy_B)=0.5$,
corresponding to excess risks $0.023$ and $0.5$, respectively.


The gap between $\policyclass_A$ and $\policyclass_B$ arises from their different \emph{misspecification structures relative to the reward}. 
The reward $r(x,y)=\mathbf 1[y\geq 0]$ depends only on the $y$-coordinate.
In $\policyclass_A$, misspecification is confined to the reward-irrelevant $x$-direction: the class fixes $\mu_x=0$ but leaves $\mu_y$ free.
Thus, the reverse-KL projection matches the expert's reward-relevant mean $\mu_y^{\mathrm e}=2$, yielding high expected reward.
In contrast, $\policyclass_B$ couples the two coordinates through $\mu_y=\mu_x$.
As a result, matching the expert distribution trades off the reward-relevant $y$-direction against the reward-irrelevant $x$-direction, leading to suboptimal expected reward.





\begin{figure}[t]
    \centering
    \begin{minipage}[t]{0.57\linewidth}
    \vspace{-3pt}
        \centering
        \includegraphics[width=0.8\linewidth]{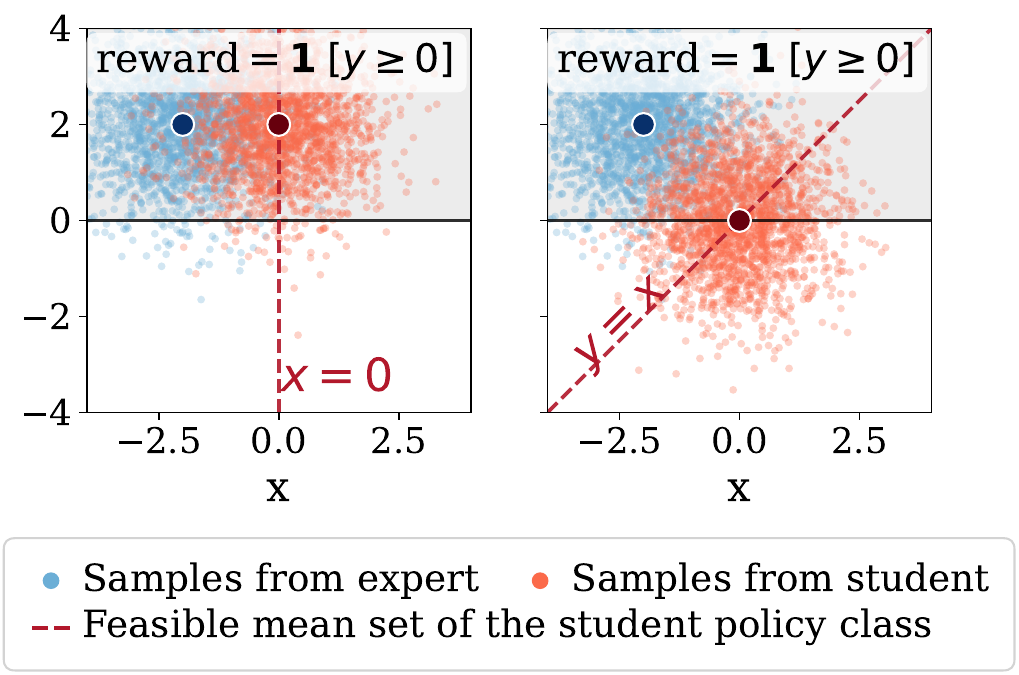}
        \caption{\
        \ifneurips
\else
        \textbf{Synthetic example: misspecification relative to reward.}
        Actions are $(x,y)\in\R^2$ and policies are Gaussians; the reward is $r(x,y)=\mathbf 1[y\ge 0]$ (shaded region marks $r(x,y)=1$).
        \fi
        \textcolor[RGB]{100,170,220}{Expert} $\expolicy=\mathcal N((-2,2),I_2)$.
        \textit{Left:} student class $\Pi_A$: $\mu_x=0$
        \ifneurips
        \else
        (red dashed line) 
        \fi
        yields a reverse-KL projection with $\val(\hat\pi_A)\approx 0.977$.
        \textit{Right:} student class $\Pi_B$: $\mu_y=\mu_x$ yields $\val(\hat\pi_B)=0.5$.
        The performance gap arises from different \textbf{misspecification structures relative to the reward}.}
        \label{fig:synthetic}
    \end{minipage}%
    \hfill%
    \begin{minipage}[t]{0.39\linewidth}
    \vspace{0pt}
        \centering
        \includegraphics[width=\linewidth]{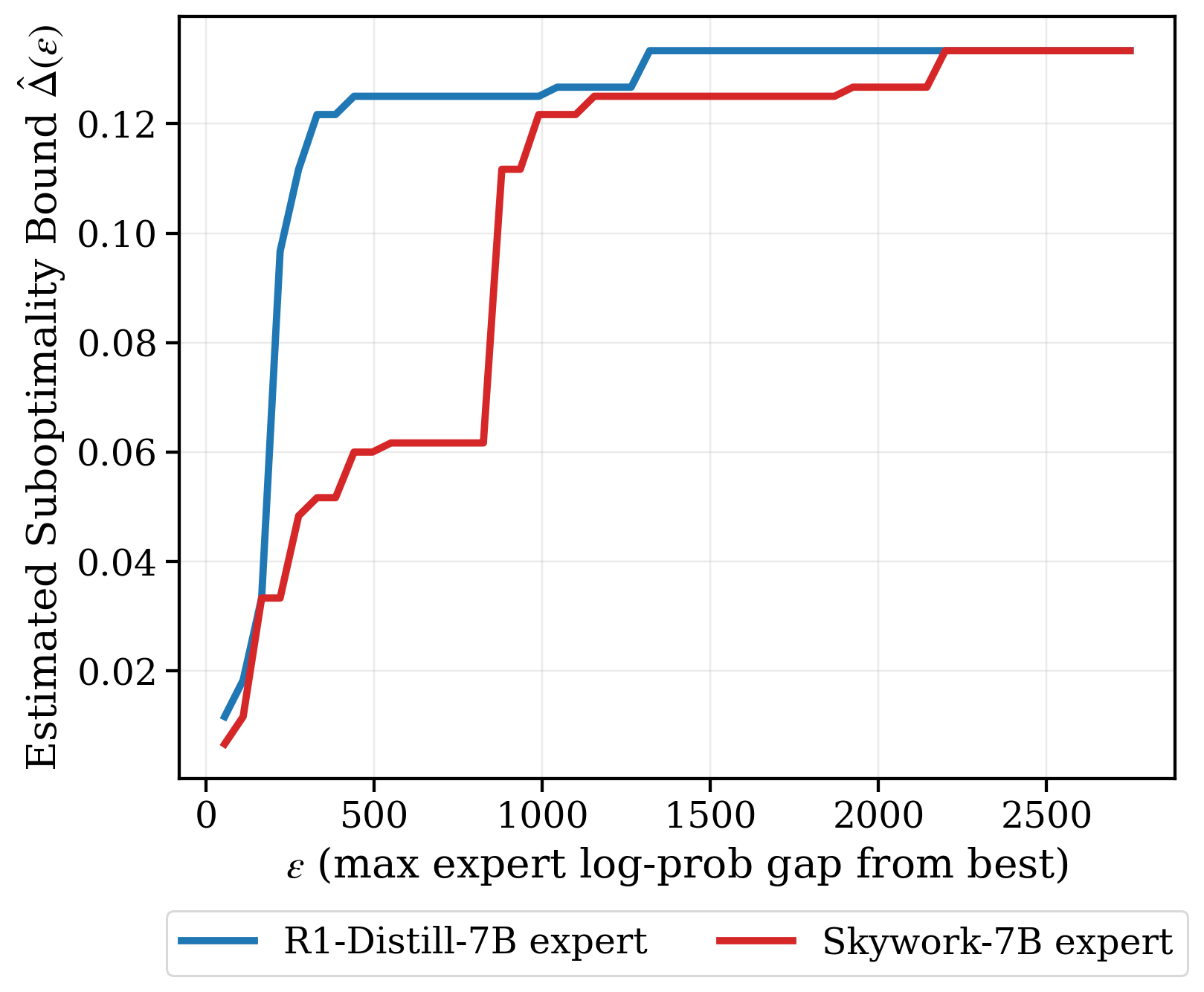}
        \caption{\textbf{Estimated suboptimality bound.}
    We empirically estimate the suboptimality bound in \Cref{thm:main} $\hat\Delta(\epsilon):=\min_{\alpha > 0}\frac{\epsilon + \hat \delta_\alpha (\epsilon)}{\alpha}$ through Monte Carlo.
    Consistent with \Cref{thm:main}, the expert leading to a better student (i.e., Skywork-7B) has a smaller suboptimality bound at every $\epsilon$.
    }
        \label{fig:minimax-suboptimality}
    \end{minipage}

\end{figure}

\ifneurips
\else
\paragraph{A General Sufficient Condition: Misspecification--Reward Alignment.}
\fi
As illustrated in the above toy example, to characterize when IL is effective in non-realizable settings, one must account for misspecification structure between the student and expert under a particular reward function.
We now propose a general condition, which we term the \textit{Misspecification--Reward Alignment Condition} for the online IL objective considered in \cref{def:onlineil}. 
Under this condition, approximate optimality for the online IL objective translates into a guarantee on the expected reward suboptimality.

\begin{tcolorbox}[
  enhanced, breakable,
  colframe=blue!40!black,
  boxrule=0.35pt, arc=1mm,
  title={
    \parbox{7.1cm}{
      \centering \textbf{Misspecification-Reward Alignment Condition}
    }
  },
  coltitle=black, fonttitle=\small\sffamily\bfseries,
  colbacktitle=blue!15!white,
  colback=blue!5!white,
  boxed title style={
    sharp corners, boxrule=0pt,
    top=1pt, bottom=0.5pt, left=0.5mm, right=0.5mm,
    borderline={0.5pt}{0pt}{blue!20!white}
  },
  attach boxed title to top left={xshift=4mm,yshift*=-1.2mm},
  boxsep=1.5mm, top=1.5mm, bottom=1.5mm, left=2.5mm, right=4mm,
  before skip=10pt, after skip=10pt
]

Without loss of generality, we can decompose the expert signal as:
\begin{align}
\label{eq:f_decomposition}
    \textstyle f(\expolicy(y\mid x)) = \alpha\, r(x,y) + h(x,y) + b(x),
\end{align}
where $\alpha>0$, $h:\mathcal X\times\mathcal Y\to\mathbb R$, and $b:\mathcal X\to\mathbb R$.
The term $\alpha r(x,y)$ is the reward-aligned component, with $\alpha$ quantifying the strength of alignment.
The residual $h(x,y)$ captures expert preferences beyond reward, and $b(x)$ is an action-independent offset.
\begin{assumption}
\label{ass:main}
Let $\optimalpolicy \in \policyclass$ be a reference policy. 
For any $\epsilon>0$, assume that for all $\policy \in \policyclass$ satisfying
$\Jon(\policy) \ge \Jon(\optimalpolicy) - \epsilon$,
\begin{align*}
    \E_{x\sim \xdist}
    \big[
        \E_{y\sim \policy(\cdot \mid x)} h(x,y)\
        -
        \E_{y\sim \optimalpolicy(\cdot \mid x)} h(x,y)
    \big]
    \le \delta_\alpha(\epsilon).
\end{align*}
\end{assumption}
\end{tcolorbox}

The decomposition in \cref{eq:f_decomposition} can be made for any fixed $\alpha>0$: for any expert signal $f(\expolicy(y\mid x))$ and any $x$-dependent offset $b_\alpha(x)$, one can define $h(x,y):=f(\expolicy(y\mid x)) - \alpha r(x,y) - b_\alpha(x)$.
Therefore, the decomposition is without loss of generality; the actual assumption lies in the alignment condition on the residual term $h$.
Crucially, Assumption~\ref{ass:main} requires that within the $\epsilon$-optimal set under $\Jon$, the residual term cannot inflate $\Jon(\policy)$ for a reward-suboptimal policy. This condition allows us to translate near-optimality for the online IL objective into a guarantee on the reward performance of the learned policy.

\ifneurips
\else

\begin{tcolorbox}[
  enhanced, breakable,
  colframe=blue!40!black,
  boxrule=0.35pt, arc=1mm,
  title={
    {
      \textbf{Effectiveness of online IL in non-realizable settings}
    }
  },
  coltitle=black, fonttitle=\small\sffamily\bfseries,
  colbacktitle=blue!15!white,
  colback=blue!5!white,
  boxed title style={
    sharp corners, boxrule=0pt,
    top=1pt, bottom=0.5pt, left=0.5mm, right=0.5mm,
    borderline={0.5pt}{0pt}{blue!20!white}
  },
  attach boxed title to top left={xshift=4mm,yshift*=-1.2mm},
  boxsep=1.5mm, top=1.5mm, bottom=1.5mm, left=2.5mm, right=4mm,
  before skip=10pt, after skip=10pt
]
\begin{theorem}
\label{thm:main}
Under \Cref{ass:main}, let $\policy \in \policyclass$ be any policy that is $\epsilon$-suboptimal for the online IL objective~\Cref{eq:onlineil}, i.e.,
$
    \textstyle \Jon(\policy)
    \ge
    \max_{\pi' \in \policyclass} \Jon(\pi') - \epsilon.
$
Then $\policy$ satisfies: \begin{align*}
    \textstyle \val(\policy)
    \ge
    \val(\optimalpolicy)
    - \frac{\epsilon + \delta_\alpha(\epsilon)}{\alpha}.
\end{align*}
\end{theorem}
\end{tcolorbox}
\fi

The proof of \Cref{thm:main} is provided in \Cref{app:proof-positive}.
To better understand when the Misspecification--Reward Alignment condition (Assumption~\ref{ass:main}) holds, we highlight two representative scenarios.
\begin{enumerate}[leftmargin=*, itemsep=0pt, parsep=0pt, topsep=0pt, partopsep=0pt]

\item 
When the reshaped expert signal $f(\expolicy(y\mid x))$ is uniformly aligned well with the reward $r(x,y)$, Assumption~\ref{ass:main} holds essentially independently of the misspecified student class.
Concretely, if $|h(x,y)|\ll \alpha$ for all $\x \in\xspace, \y \in \yspace$, then the suboptimality bound satisfies $\frac{\epsilon+\delta(\epsilon)}{\alpha}\ll 1$ as $\epsilon\to 0$.


    \item 
Although the expert signal $f(\expolicy(y\mid x))$ encodes preferences beyond the reward, captured by the residual $h(x,y)$, the misspecified student is unable to represent these preferences.
For example, the expert may prefer more concise responses, while any student policy that achieves high performance requires longer chains of reasoning due to limited capability.
In this case, the reward-irrelevant preferences (e.g., response length) in the expert signal do not substantially distinguish among near-optimal policies in the student class.
Formally, for any $\policy_1,\policy_2\in\policyclass$ satisfying
$\Jon(\policy_i)\ge \max_{\pi\in\policyclass}\Jon(\pi)-\epsilon$ for $i=1,2$, we have
\[
\textstyle\left|
\E_{x\sim\xdist}\!\left[
\E_{y\sim \policy_1(\cdot \mid x)} h(x,y)
-
\E_{y\sim \policy_2(\cdot \mid x)} h(x,y)
\right]
\right|
\ll 1.
\]
This regime also explains why the expert signal is effective for $\policyclass_A$ in the 2D Gaussian example.

\end{enumerate}




\ifneurips
\else

It is worth noting the comparison between \Cref{thm:main} and the classical view of online IL as mitigating error accumulation. 
A classic account is that online interaction mitigates the compounding of per-step prediction errors over long horizons, leading to DAgger-style algorithms \citep{ross2010efficient, ross2011dagger, lu2025onpolicydistillation}. 
In contrast, our work studies the \emph{information bottleneck} that arises in non-realizable settings, an orthogonal limitation of offline imitation learning that already appears in contextual bandits, i.e., when $H=1$. 
We further characterize how the reward function interacts with the misspecification structure to enable online imitation learning in non-realizable settings, a perspective absent from previous DAgger-style analyses.
\fi

\ifneurips
\else
\paragraph{Finite-Sample Guarantee.}
\fi
\ifneurips
\else
The efficient training of online and on-policy algorithms for LLM post-training (including RL and on-policy distillation) crucially relies on a strong base model, so that the model can receive positive feedback with nontrivial probability \citep{deepseekai2025deepseekr1incentivizingreasoningcapability, zeng2025simplerl}. This requirement can be characterized theoretically through \emph{coverage coefficients} defined with respect to the base policy (e.g., \citet{xieexploratory,huang2025sharpening, foster2025good, huangcorrecting}):
\fi
\ifneurips
\else
\begin{align*}
    \basecoveragestar(\optimalpolicy)&\textstyle \coloneqq \sup_{\x\in\xspace,\y\in\yspace} \frac{\optimalpolicy(\y|\x)}{\basepolicy(\y|\x)}, \ \basecoverage_2(\optimalpolicy)\coloneqq \mathbb E_{\x \sim \xdist} \mathbb E_{\y\sim \basepolicy} \left(\frac{\optimalpolicy(\y|\x)}{\basepolicy(\y|\x)}\right)^2,
\end{align*}
\fi
where $\optimalpolicy\in\policyclass$ is a reference policy (e.g., an optimal student policy).
\ifneurips
\else

Besides coverage, we also require a boundedness condition on the
shaped expert signal, which plays the same role as the standard bounded-reward
condition in finite-sample RL and off-policy evaluation analyses.

\begin{assumption}[Bounded shaped expert signal]
\label{ass:bounded-signal}
For all $\x\in\xspace$ and $\y\in\yspace$,
\[
    |f(\expolicy(\y\mid\x))| \leq \vmax .
\]
\end{assumption}

Using standard pessimism techniques \citep{swaminathan2015self, wang2024oracle, gabbianelli2024importance}, one can extract a finite-sample result for the objective \Cref{eq:onlineil} under such a condition. Below we present the finite-sample guarantee, and defer details and proof to  \Cref{app:finite-sample}.
\fi

\ifneurips
\else
\begin{theorem}
Suppose \cref{ass:bounded-signal}.
For a base model $\basepolicy$, a reference student policy $\optimalpolicy$, with probability $1-\delta$, the learned policy $\hat \pi$ in \cref{eq:ipw} satisfies
\begin{equation*}
\begin{aligned}
    \Jon(\optimalpolicy)-\Jon(\policyhat) \lesssim  
     \sqrt{\frac{\basecoverage_2(\optimalpolicy)V_{\max}^2\log (|\policyclass|/\delta)}{N}} +  \frac{\basecoverage_\infty(\optimalpolicy)V_{\max}\log (|\policyclass|/\delta)}{N}.
\end{aligned}
\end{equation*}
\label{thm:finite-sample}
\end{theorem}
\fi

Compared to the lower bound for offline IL (\Cref{thm:lowerboundoffline}), the coverage dependence improves from $\excoverage$ to $\basecoveragestar(\optimalpolicy)$.
Since the base policy is the model before fine-tuning, it is expected to cover policies reachable during training better than an external expert. 
While a strong base model appears necessary for efficient online algorithms, it remains open whether its role can be characterized by a finer notion than these coverage coefficients.

\ifneurips
\else

\fi
\ifneurips
\else
We also note that the hard instances in \cref{thm:lowerboundoffline} can also be constructed to satisfy \cref{ass:main} and \cref{ass:bounded-signal}, yielding a clean offline--online separation within the same regime. See \cref{rem:separation} for further discussion.
\fi

\ifneurips
\else
\paragraph{Empirical Validation.}
\fi
We further provide an empirical validation of the misspecification--reward alignment condition (\cref{ass:main}). Following the setup of \citet{li2026rethinkingopd}, we start from the same base model, DeepSeek-R1-Distill-Qwen-1.5B \citep{deepseekai2025deepseekr1incentivizingreasoningcapability}, and perform on-policy distillation with two different experts: DeepSeek-R1-Distill-Qwen-7B (R1-Distill-7B) \citep{deepseekai2025deepseekr1incentivizingreasoningcapability} and Skywork-OR1-Math-7B (Skywork-7B) \citep{he2025skywork}. On AIME 2024 and AIME 2025, distillation from Skywork-7B improves the student's average performance from 25.0\% to 32.8\%, whereas distillation from R1-Distill-7B improves it to 27.1\% (averaged over three runs).


To empirically assess \cref{ass:main}, we estimate the suboptimality bound in
\cref{thm:main}. The decomposition in \cref{eq:f_decomposition} is not unique: each choice of
\(\alpha > 0\) induces a residual upper bound \(\delta_\alpha(\epsilon)\). Since the bound
in \cref{thm:main} depends on this choice through
\((\epsilon+\delta_\alpha(\epsilon))/\alpha\), we further minimize the estimated bound over $\alpha$.
We estimate the bound using the following procedure. 
For each expert, we perform three on-policy distillation runs.  We collect checkpoints from all distillation trajectories across the two experts, with checkpoints saved every 10 steps, and use their union as a proxy policy class ($\widehat{\Pi}$).
We take the checkpoint with the highest empirical reward as
the reference policy \(\widehat{\pi}^{\star}\). Then, for each expert separately and for each \(\epsilon > 0\), we
consider the set of checkpoints whose empirical online IL objective is within
\(\epsilon\) of that of \(\widehat{\pi}^{\star}\):
$\widehat{\Pi}_\epsilon:=
\left\{
\pi \in \widehat{\Pi}:
\widehat{J}_{\mathrm{on}}(\pi)
\ge
\max _{\pi'\in {\widehat \Pi}}\widehat{J}_{\mathrm{on}}(\pi' )-\epsilon
\right\}
$.
We estimate the residual upper bound $\delta_\alpha(\epsilon)$ in \cref{ass:main} over $\widehat{\Pi}_\epsilon$ by Monte Carlo, and denote the resulting estimate by $\widehat{\delta}_\alpha(\epsilon)$.
We then report the estimated bound
\[
\widehat{\Delta}(\epsilon)
:=
\min_{\alpha>0}
\frac{\epsilon+\widehat{\delta}_\alpha(\epsilon)}{\alpha}.
\]
Details are provided in \Cref{app:detail_onlineil_condition}.

As shown in \Cref{fig:minimax-suboptimality}, Skywork-7B yields a much smaller estimated bound than R1-Distill-7B across~$\epsilon$, consistent with its larger distillation improvement. The relatively large range of $\epsilon$ is expected, since the online IL objective is based on sequence-level log probabilities, whose scale grows approximately linearly with response length.

\section{Conclusion and Discussion}
\label{sec:discussion}
\ifneurips
\else
In this work, we show that misspecification is a key factor driving the empirical performance gap between offline and online IL methods in LLM post-training. Our results help clarify what leads to the empirical performance gap between offline and online IL methods in LLM post-training, and suggest several directions that may be worth exploring in future work.
Empirically, in realizable settings, offline IL already recovers expert performance, leaving no room for online interaction to help. In this sense, the gap often observed between SFT and online methods such as RLVR or on-policy distillation does not necessarily come from some inherent algorithm properties. Instead, it can arise from the offline data distribution induced by a non-realizable expert, which may be poorly suited for the target student class. 

On the theoretical front, we characterize an information-theoretic limitation of offline IL under misspecification, and provide structural conditions under which online IL can still be effective even when the discrepancy between expert and student is large. This perspective complements prior explanations based on error accumulation. An important question not fully explained by the current framework is why on-policy distillation can empirically achieve better training efficiency than RLVR. Answering this question may require analyzing the multi-step structure of language generation and characterizing how dense reward signals improve learning efficiency in this setting.
\fi

\section*{Acknowledgments}

The authors thank Audrey Huang, Abhishek Panigrahi, Aditi Raghunathan, Dhruv Rohatgi, Kaiyue Wen and Chen Wu for valuable discussions and feedback. B.L. is supported by the Kempner Fellowship from the Kempner Institute at Harvard. A.R. and J.G. are supported in part by NSF awards IIS-2211907, CCF-2238523, IIS-2403275, an Amazon Research
Award, ONR award N000142512124, a Google Research Scholar Award, and an OpenAI Superalignment Fast Grant.

\bibliography{references}
\bibliographystyle{icml2026}

\newpage
\appendix
\section{Proofs and Additional Results}

\subsection{Proof of \Cref{thm:lowerboundoffline}}
\label[app]{app:proof-lowerbound}

Without loss of generality, assume that $K := |\Pi| = 2^l$ for some
$l \in \mathbb N^*$. If not, let
$K' = 2^{\lfloor \log_2 K \rfloor} \le K$. We construct the hard instance
using only $K'$ policies and augment the policy class with $K-K'$ dummy
policies. Since the original hard instance is embedded in this larger class,
this augmentation cannot decrease the minimax risk.

To prove the lower bound of the optimal suboptimality, we construct a joint distribution $\mathcal P$ over instances
$\cbinstance$ and experts $\expolicy$ such that for any offline learner $\hat\pi(D)$
based on $N$ i.i.d.\ expert demonstrations $D$,
\begin{align*}
\E_{(\expolicy, \cbinstance)\sim \mathcal P} 
    \left[
    1-\mathbb E\left[\val (\hat \pi )\right ] 
    \right]
    \gtrsim 
    \min\!\left\{1,\ \frac{\excoverage \log |\policyclass|}{N}\right\}.
\end{align*}


Let $\xspace=\{x_1,\ldots,x_l\}$ and $\yspace=\{y_1,y_2,y_3\}$. 
Define the student policy class $\Pi$ to consist of deterministic policies that,
for each context $x_i$, choose either $y_2$ or $y_3$:
\[
\forall i\in[l], \ \pi(\cdot\mid x_i)\in\{\delta_{y_2},\delta_{y_3}\}.
\]
Thus $|\Pi|=2^l=K$. Let the context distribution $\xdist$ satisfy
$\xdist(x_i)=\zeta $ for $i\le l-1$ and $\xdist(x_l)=1-(l-1)\zeta$. Moreover, let $\mathcal S := \{y_2,y_3\}^l$, and index environments by $s=(s_1,\ldots,s_l)\in\mathcal S$.
For each $s\in\mathcal S$, define the expert policy $\expolicy_s$ by, for all $i\in[l]$,
\[
\expolicy_s(\cdot \mid x_i) \;=\; \Big(1-\tfrac{1}{\excoverage}\Big)\delta_{y_1} \;+\; \tfrac{1}{\excoverage}\,\delta_{s_i}.
\]
Define the reward function $r_s$ by
\[
r_s(x_i,y) \;=\; \mathbf{1}\!\left[y=y_1 \ \text{or}\ y=s_i\right].
\]
The optimal student policy is $\optimalpolicy_s(\cdot\mid x_i)=\delta_{s_i}$ for all $i\in[l]$.

Let $\mathcal P$ denote the uniform prior over instances induced by $\mathcal S$: draw $S\sim\mathrm{Unif}(\mathcal S)$ and set the instance to be $(\xdist,r_S,\expolicy_S)$.
Let $\xspace(D)$ be the context for which the secret action $s_i$ is covered by i.i.d. samples in dataset $D$. Then 
\begin{align*}
    \E_{(\expolicy, \cbinstance)\sim \mathcal P }
    \E[1-\val(\policyhat)]
    &=
    \E\E_{(\expolicy, \cbinstance)\sim \mathcal P }[1-\val(\policyhat)]\\
    & \geq
    \E\E_{(\expolicy, \cbinstance)\sim \mathcal P }
    \left[\frac \zeta 2
    \sum_{i=1}^ {l-1}\mathbf 1\left[i\not \in \xspace(D)\right]
    \right]\\
    &  \geq
    \frac 12 \mathbb E \left[
    \sum_{i=1}^{l-1}\zeta \Pr\left[i\not \in \xspace(D)\right]
    \right].
\end{align*}

For each $i\le l-1$, the probability that $i$ becomes covered in one sample is
$\Pr(x=x_i, y\neq y_1)=\zeta/\excoverage$. Setting $\zeta =\min \{\frac 1l, \excoverage/N\}$ gives
\[
\Pr(i\notin\xspace(D))=\left (1-\frac{\zeta }{\excoverage}\right)^N \gtrsim 1,
\]
and therefore
\begin{align*}
    \E_{(\expolicy, \cbinstance)\sim \mathcal P }
    \E[1-\val(\policyhat)] \gtrsim \zeta l.
\end{align*}
Plugging back yields
\[ 
\E_{(\pi^e,\cbinstance)\sim\mathcal P}\E_D\!\left[1-\val(\hat\pi(D))\right]
\gtrsim \zeta l 
\asymp \min \left\{1,\frac{\excoverage  \log |\policyclass|}{N }\right\}
\]
as was to be shown. \qed

\subsection{Proof of \Cref{thm:main}}
\label[app]{app:proof-positive}

Fix any $\policy \in \policyclass$ that is $\epsilon$-suboptimal for the online IL objective~\cref{eq:onlineil}, i.e.,
\begin{align*}
\Jon(\policy) \ge \max_{\policy'\in \policyclass} \Jon(\policy')-\epsilon
\geq \Jon(\optimalpolicy) - \epsilon.
\end{align*}
By Assumption \ref{ass:main}, there exist $\alpha>0$, $h$, and $b$ such that for all $(x,y)$,
\[
f(\expolicy(y\mid x)) = \alpha\, r(x,y) + h(x,y) + b(x).
\]
Taking expectation over $x\sim\xdist$ and $y\sim \policy(\cdot\mid x)$ yields
\begin{align*}
\Jon(\policy)
&= \E_{x\sim\xdist}\E_{y\sim\policy(\cdot\mid x)} f(\expolicy(y\mid x))\\
&= \alpha\, \E_{x\sim\xdist}\E_{y\sim\policy(\cdot\mid x)} r(x,y)
  + \E_{x\sim\xdist}\E_{y\sim\policy(\cdot\mid x)} h(x,y)
  + \E_{x\sim\xdist} b(x)\\
&= \alpha\, \val(\policy)
  + \E_{x\sim\xdist}\E_{y\sim\policy(\cdot\mid x)} h(x,y)
  + \E_{x\sim\xdist} b(x).
\end{align*}
The same decomposition holds for $\optimalpolicy$:
\[
\Jon(\optimalpolicy)
= \alpha\, \val(\optimalpolicy)
  + \E_{x\sim\xdist}\E_{y\sim\optimalpolicy(\cdot\mid x)} h(x,y)
  + \E_{x\sim\xdist} b(x).
\]
Subtracting the two identities cancels the $b(x)$ term and gives
\begin{align*}
\Jon(\policy) - \Jon(\optimalpolicy)
= \alpha\big(\val(\policy)-\val(\optimalpolicy)\big)
+ \E_{x\sim\xdist}\!\Big[
\E_{y\sim\policy(\cdot\mid x)} h(x,y)
-\E_{y\sim\optimalpolicy(\cdot\mid x)} h(x,y)
\Big].
\end{align*}
Rearranging,
\begin{align*}
\alpha\big(\val(\optimalpolicy)-\val(\policy)\big)
=
\Jon(\optimalpolicy)-\Jon(\policy)
+ \E_{x\sim\xdist}\!\Big[
\E_{y\sim\policy(\cdot\mid x)} h(x,y)
-\E_{y\sim\optimalpolicy(\cdot\mid x)} h(x,y)
\Big].
\end{align*}
Now use $\Jon(\policy)\ge \Jon(\optimalpolicy)-\epsilon$, which implies
$\Jon(\optimalpolicy)-\Jon(\policy)\le \epsilon$.
Moreover, since $\Jon(\policy)\ge \Jon(\optimalpolicy)-\epsilon$, the second part of
\cref{ass:main} applies and yields
\[
\E_{x\sim\xdist}\!\Big[
\E_{y\sim\policy(\cdot\mid x)} h(x,y)
-\E_{y\sim\optimalpolicy(\cdot\mid x)} h(x,y)
\Big]\le \delta(\epsilon).
\]
Combining these two bounds,
\[
\alpha\big(\val(\optimalpolicy)-\val(\policy)\big)
\le \epsilon + \delta(\epsilon),
\]
and dividing by $\alpha>0$ finishes the proof. \qed

\subsection{Finite Sample Guarantee of Learning Objective \Cref{eq:onlineil}}
\label[app]{app:finite-sample}

The online IL objective \cref{eq:onlineil} is typically optimized in an online and on-policy manner in LLM post-training, as in on-policy distillation \citep{gu2023minillm, lu2025onpolicydistillation}, where the learner adaptively generates responses according to the current student policy (similarly to reinforcement learning).  In practice, it is often optimized in an RL framework.

The sample complexity analysis of on-policy algorithms for a general policy class is challenging. 
Existing results typically require additional assumptions, such as reverse-KL regularization and realizability of the policy class, or study simplified policy classes  \citep{xieexploratory, foster2025good, kim2026coverage}. 
Other online contextual-bandit results study indirect learning formulations, where policies are learned through reward/value-function estimation rather than direct policy learning \citep{foster2018practical, foster2020beyond, simchi2022bypassing}.
These results do not directly apply to our setting, where the learner optimizes over a general policy class. Therefore, we instead present a simpler algorithm based on \emph{pessimistic policy optimization} \citep{swaminathan2015self, wang2024oracle, gabbianelli2024importance}, which only samples from the base policy and queries expert feedback, yet already yields meaningful improvements over offline imitation learning. 

Concretely, the learner draws $N$ i.i.d.\ samples with
$x_i \sim \xdist$ and $y_i \sim \basepolicy(\cdot \mid x_i)$, and queries the expert signal $f(\expolicy(y_i\mid x_i))$.
For any policy $\policy \in \policyclass$, define the inverse-propensity weighted estimator
\begin{align*}
 \hatJon(\policy)
\;\coloneqq\;
\frac{1}{N}\sum_{i=1}^N
\frac{\policy(y_i \mid x_i)}{\basepolicy(y_i \mid x_i)}\, f(\expolicy(\y_i\mid \x_i)),
\end{align*}
and the regularizer 

\begin{align*}b(\pi)\coloneqq\sqrt{\frac{2\basecoverage_2(\policy)V_{\max}^2\log (2|\policyclass|/\delta)}{N}} +  \frac{2\basecoverage_\infty(\policy)V_{\max}\log (2|\policyclass|/\delta)}{3N},
\end{align*}
where \begin{align*}
    \basecoverage_\infty (\policy)&\coloneqq \sup_{\x\in\xspace,\y\in\yspace} \frac{\policy(\y|\x)}{\basepolicy(\y|\x)}, \ \basecoverage_2(\policy)\coloneqq \mathbb E_{\x \sim \xdist} \mathbb E_{\y\sim \basepolicy} \left(\frac{\policy(\y|\x)}{\basepolicy(\y|\x)}\right)^2.
\end{align*}

The learned policy $\policyhat$ is defined as 
\begin{align}
\label{eq:ipw}
    \policyhat = \argmax_{\policy\in\policyclass} \hatJon(\policy) - b(\pi).
\end{align}

Note that $b(\pi)$ can be evaluated without querying the expert, and therefore does not contribute to the statistical complexity of expert feedback.

We now prove \cref{thm:finite-sample}.

\begin{proof}[Proof of \cref{thm:finite-sample}]
For a policy $\policy$, denote $w^\policy (\x,\y) = \frac{\policy(\y\mid \x)}{\basepolicy(\y\mid \x)}$, $\Zi=w^\policy (\x_i,\y_i)f(\expolicy(\y_i\mid\x_i))$. Then 

\begin{align*}
    \hatJon(\policy)&=\frac 1N \sum_{i=1}^N \Zi,\\
    \Jon(\pi)&=\E[\Zi].
\end{align*}
So
\begin{align*}
    \hatJon(\policy)-\Jon(\policy)=\frac 1 N \sum_{i=1}^N (\Zi -\mathbb E \Zi).
\end{align*}
Since 
\begin{align*}
&|\Zi|\leq w^\policy (\x,\y)|f(\expolicy(\y_i\mid\x_i))|\leq V_{\max} \basecoveragestar(\policy), \\
&\Var(\Zi)\leq \E [(\Zi)^2]=\E _{x\sim \xdist} \E _{y\sim \basepolicy(\cdot \mid \x)} [ w^\policy (\x,\y)^2V_{\max}^2]\leq V_{\max}^2\basecoverage_2(\policy),
\end{align*}
Bernstein inequality gives
\begin{align*}
    \Pr
    \left(
    \left|
    \hatJon(\policy)-\Jon(\policy) 
    \right|
    > b(\pi)
    \right)
    \leq \frac{\delta }{|\policyclass|}.
\end{align*}
By union bound, with probability at least $1-\delta$, it holds for any $\pi\in \policyclass$ that
\begin{align*}
    \left|
    \hatJon(\policy)-\Jon(\policy) 
    \right|
    \leq b(\pi).
\end{align*}

Under this event, 
\begin{align*}
    \hatJon(\policyhat)\leq \Jon (\policyhat) + b(\policyhat),\\
    \Jon(\optimalpolicy)\leq \hatJon(\optimalpolicy) + b(\optimalpolicy).
\end{align*}
Thus
\begin{align*}
    \Jon(\policyhat) \geq \hatJon(\policyhat) - b(\policyhat)\geq \hatJon(\optimalpolicy) - b(\optimalpolicy)\geq \Jon(\optimalpolicy) - 2b(\optimalpolicy).
\end{align*}
\end{proof}

\begin{remark}[Online--Offline Separation under the Structural Conditions]
\label{rem:separation}
We note that the hard instances used in the offline lower bound
(\cref{thm:lowerboundoffline}) can also be chosen to satisfy the additional
conditions required by the online upper bound, including \cref{ass:main} and
the $V_{\max}$-bounded signal condition, by taking a truncated identity shaping
function and choosing $\alpha$ appropriately.

Indeed, consider the hard instance indexed by $s$ in the proof of
\cref{thm:lowerboundoffline}. Recall that for each context $x_i$, the expert
policy is given by
\[
\expolicy_s(y\mid x_i)
=
\begin{cases}
1-\frac{1}{\excoverage}, & y=y_1,\\
\frac{1}{\excoverage}, & y=s_i,\\
0, & \text{otherwise}.
\end{cases}
\]
Take the shaping function $f:\mathbb R\to\mathbb R$ to be the truncated identity
\[
f(p) := \min\left\{p,\frac{1}{\excoverage}\right\}.
\]
Applying this shaping function to the expert probability gives
\[
f(\expolicy_s(y\mid x_i))
=
\begin{cases}
\frac{1}{\excoverage}, & y=y_1,\\
\frac{1}{\excoverage}, & y=s_i,\\
0, & \text{otherwise}.
\end{cases}
\]
Since student policies only choose between $y_2$ and $y_3$, the relevant signal
values are
\[
f(\expolicy_s(s_i\mid x_i))=\frac{1}{\excoverage},
\qquad
f(\expolicy_s(y\mid x_i))=0
\quad \text{for } y\in\{y_2,y_3\}\setminus\{s_i\}.
\]
Thus the signal gap between the optimal student action and the suboptimal
student action is exactly $1/\excoverage$, while the corresponding reward gap is
$1$.

Thus  $f(\expolicy_s(y\mid x))=\frac 1 {\excoverage}r(x,y)$. 
Hence \cref{ass:main} holds with $\alpha=1/\excoverage$ and $\delta(\epsilon)=0$. Moreover, the shaped signal is
$V_{\max}$-bounded with $V_{\max}=1/\excoverage$.

Therefore, converting the online estimation error $\epsilon$ into suboptimality
through $\epsilon/\alpha$ gives
\begin{align*}
1-\val(\policyhat_{\mathrm{on}})
&\lesssim
\frac{1}{\alpha}
\left(
\sqrt{
\frac{
\basecoverage_2(\optimalpolicy) V_{\max}^2
\log(2|\policyclass|/\delta)
}{N}
}
+
\frac{
\basecoverage_\infty(\optimalpolicy) V_{\max}
\log(2|\policyclass|/\delta)
}{N}
\right) \\
&=
\sqrt{
\frac{
\basecoverage_2(\optimalpolicy)
\log(2|\policyclass|/\delta)
}{N}
}
+
\frac{
\basecoverage_\infty(\optimalpolicy)
\log(2|\policyclass|/\delta)
}{N}.
\end{align*}
In particular, after substituting
$\alpha=V_{\max}=1/\excoverage$, the online guarantee no longer scales with the
expert-coverage coefficient $\excoverage$. By contrast,
\cref{thm:lowerboundoffline} shows that any offline IL algorithm must incur
\[
1-\mathbb E[\val(\policyhat)]
\gtrsim
\min\left\{1,\frac{\excoverage \log |\policyclass|}{N}\right\}.
\]
Thus, even within the regime where \cref{ass:main} and the bounded-signal
condition hold, offline IL must pay the expert-coverage dependence, whereas
online IL depends instead on the typically more favorable base-policy coverage
terms. This gives a clean offline--online separation under the structural
conditions.

Although the online IL guarantee contains a slower term
$\sqrt{\basecoverage_2(\optimalpolicy)/N}$, the coverage coefficient appearing
there, $\basecoverage_2(\optimalpolicy)$, can be much smaller than the
worst-case coverage coefficient $\basecoverage_\infty(\optimalpolicy)$.
\end{remark}
\newpage
\section{Additional Experiments}
\label[app]{app:additional_experiments}

\ifneurips
\else
\paragraph{Out-of-distribution generalization and catastrophic forgetting.}
\fi
We evaluate the model's out-of-distribution (OOD) generalization on the Countdown task, where test instances use a larger number range than those seen during training. We also assess catastrophic forgetting during GSM8K training by measuring general language capability on the MMLU benchmark \citep{hendryckstest2021}.
\Cref{fig:ood-dist} shows that both SFT and on-policy distillation with a realizable expert achieve OOD performance matching that of the expert, whereas SFT with a non-realizable expert performs substantially worse and exhibits a larger gap between in-distribution and OOD performance.
We observe a similar pattern for forgetting: under SFT with a realizable expert or on-policy distillation, the MMLU score decreases only slightly throughout training, with a final drop within $0.5$ points; in contrast, under a non-realizable expert, it drops more substantially, by more than $2$ points.

These findings refine prior observations that SFT tends to generalize worse OOD \citep{chu2025sft} and forget more than RL \citep{shenfeld2025rlrazor, chen2025retaining}. Instead, our results suggest this gap is not inherent to the algorithm, but arises primarily under non-realizability. It is the training data distribution, rather than the SFT algorithm itself, that drives inferior OOD generalization and severe forgetting.

\ifneurips
\else
\begin{figure}[h]
    \centering
    \includegraphics[width=0.7\linewidth]{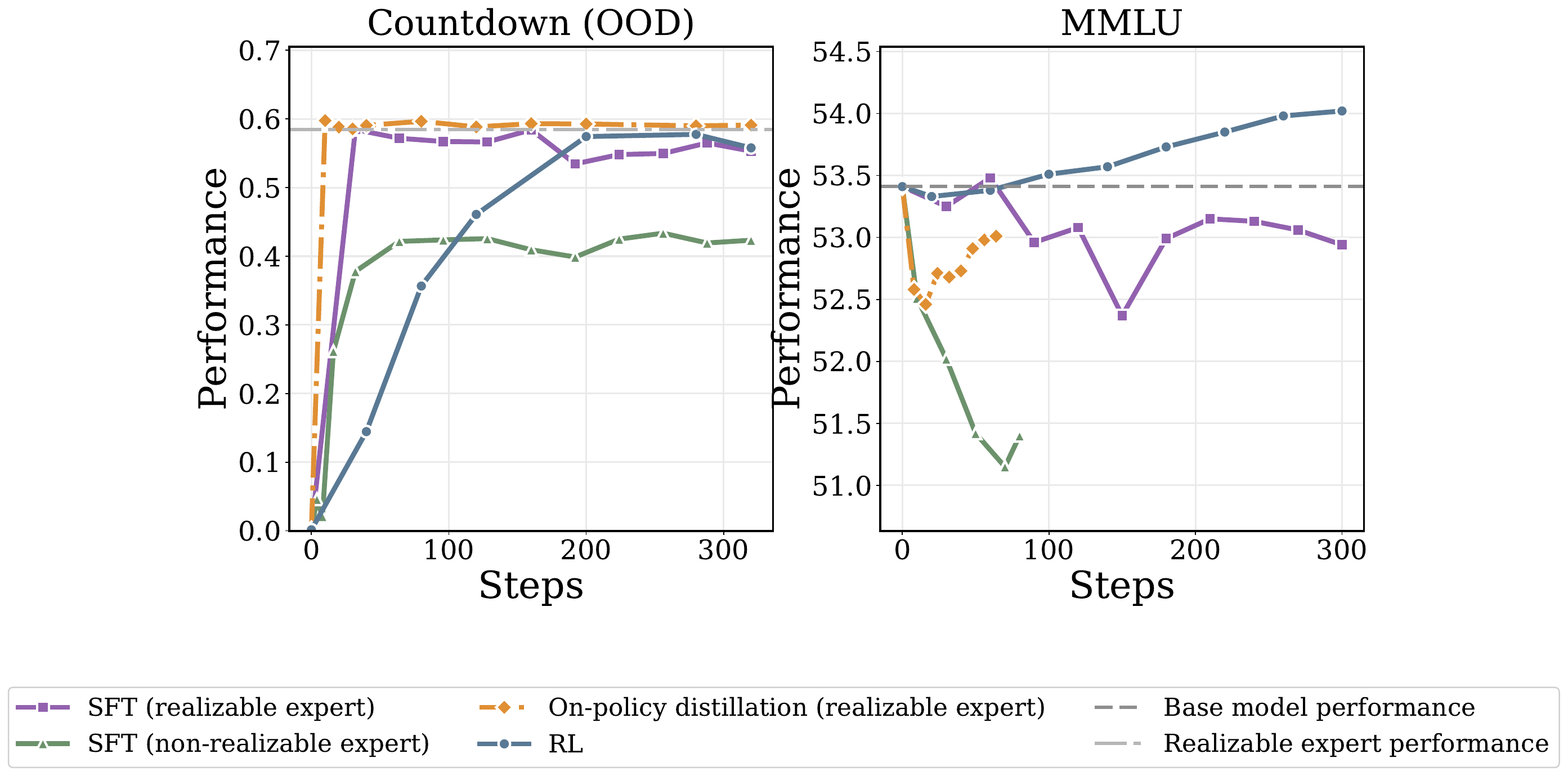}
    \caption{\textbf{OOD generalization and catastrophic forgetting.} Left: Models are evaluated on Countdown instances with larger number range. Right: Catastrophic forgetting during GSM8K training. Models are evaluated on MMLU benchmark. 
    Both on-policy distillation and SFT from a realizable expert achieve strong OOD performance that matches the expert, whereas SFT from a non-realizable expert yields much worse OOD generalization and exhibits more severe forgetting.}
    \label{fig:ood-dist}
\end{figure}
\fi

\section{Experimental Details}
\label[app]{app:experiment_details}
\subsection{Countdown Task}
\label[app]{sec:appendix_experiment_detail_countdown}

\begin{responsebox}{Countdown Task Example}
\fontfamily{\sfdefault}\selectfont 

\textbf{[INST]} Using the numbers [5, 94, 9, 44], create an equation that equals 93. You can use basic arithmetic operations (+, -, *, /) and each number can only be used once. Show your work in \verb|<think>| \verb|</think>| tags. And return the final answer in \verb|<answer>| \verb|</answer>| tags, for example \verb|<answer>|(1 + 2) / 3\verb|</answer>|. \textbf{[/INST]}

Let me solve this step by step.
\end{responsebox}

\paragraph{Experimental Setup.}

We utilize the dataset introduced by \citet{tinyzero}, which comprises $327{,}680$ training instances and $1{,}024$ test samples.\footnote{\url{https://huggingface.co/datasets/Jiayi-Pan/Countdown-Tasks-3to4}} To illustrate the task format, we provide a representative training prompt.

Our implementation is built upon the official codebase of \citet{tinyzero},\footnote{\url{https://github.com/Jiayi-Pan/TinyZero}} incorporating specific adaptations for NVIDIA A100 execution from a community fork.\footnote{\url{https://github.com/JerryWu-code/TinyZero}}

\paragraph{RL Training Details.}

We follow the training setup of \citet{tinyzero} and provide the details below for completeness.
We fine-tune the \texttt{Qwen2.5-3B-Instruct} model \citep{team2024qwen2} using Reinforcement Learning (RL) for $320$ steps on $2$ NVIDIA A100 GPUs. The training configuration employs a global batch size of $128$ with $5$ rollouts per prompt, and a mini-batch size of $64$ for optimization. We set the learning rate to $1\times10^{-6}$ and the KL penalty coefficient to $\beta_{\mathrm{KL}}=1\times10^{-3}$. The maximum generation length is capped at $1024$ tokens. We use a reward function that assigns $1.0$ to correct responses, $0.1$ to incorrect but format-compliant responses, and $0$ to invalid outputs.



\begin{table}[h]
\caption{RL training details for the Countdown task.}
\centering
\renewcommand{\arraystretch}{1.15}
\setlength{\tabcolsep}{14pt}
\small

\begin{tabular}{ll}
\toprule
\textbf{Parameter} & \textbf{Value} \\
\midrule
Max prompt length & 2,048 \\
Clip ratio & 0.2 \\
Validation temperature & 1 \\
Rollout temperature & 1 \\
Validation top-$p$ & 0.7 \\
Validation top-$k$ & 50 \\
\bottomrule
\end{tabular}
\end{table}



\paragraph{SFT Training Details.}

For training with both realizable and non-realizable experts, we use the same SFT configuration.
We perform one-pass supervised fine-tuning with LlamaFactory on \texttt{Qwen2.5-3B} using full-parameter training.
The model is trained for $3{,}200$ optimization steps with a maximum sequence length of $2{,}048$ tokens.
We use a per-device batch size of $16$, a learning rate of $1\times10^{-5}$, a cosine learning rate schedule, and a warmup ratio of $0.1$. For each question, we generate $5$ different solutions, matching our RL setup, where $5$ rollouts are generated for each question. For SFT with a realizable (RL-trained) expert, we roll out the expert with temperature $1.0$ and disable both top-$k$ and top-$p$ filtering.

For the GSM8K and DeepScaleR experiments, we use the same SFT configuration, except that the batch size and maximum sequence length are matched to the corresponding RL setting. We therefore do not repeat these details separately.

One RL training step consists of multiple gradient updates.
In \cref{fig:in-dist}, we report progress in terms of RL steps.
To align SFT with RL on the horizontal axis, we rescale SFT steps by a constant factor
(/10 for Countdown and /16 for GSM8K and DeepScaleR).


  

    
    

\paragraph{On-Policy Distillation Training Details.}
  We use TRL's MiniLLM trainer \citep{vonwerra2020trl, gu2023minillm}.
  We train for $3{,}200$ optimization steps on 8 GPUs.
  We use a per-device batch size of $4$ and a global batch size of $64$, which corresponds to gradient accumulation over 2 steps.
  For each prompt, the student samples 4 completions on-policy.
  The maximum prompt length is $256$ tokens and the maximum completion length is $1{,}024$ tokens.
  We use a learning rate of $1\times10^{-5}$.
  Unless otherwise noted, other optimization settings follow the default TRL MiniLLM configuration, including AdamW, a linear learning-rate schedule with no warmup, and dropout disabled.

For the GSM8K experiments, we use a learning rate of $5\times 10^{-6}$, while for DeepScaleR we use $1\times 10^{-5}$. 
The sequence length in each setup is matched to that used in the corresponding RL training. 
All other configurations follow the Countdown setup.

\paragraph{Evaluation Details.}

During evaluation, we sample with temperature $1.0$, top-$p$ $0.7$, and top-$k$ $50$.
For each test example, we generate $8$ responses and report the averaged accuracy.
We set the maximum generation length to $1024$ tokens.

\paragraph{OOD Evaluation Task.}
In Section~\ref{sec:compare_realizable} and \cref{app:additional_experiments}, we evaluate both in-distribution (ID) and out-of-distribution (OOD) performance for the Countdown task. For OOD evaluation, we use instances with four operands only, with larger input and answer ranges.

  

    
    

\begin{table}[h]
\caption{Detailed settings for Countdown tasks used in ID and OOD evaluation.}
\label{tab:setup_comparison}
\centering
\renewcommand{\arraystretch}{1.1}
\setlength{\tabcolsep}{8pt}
\small

\begin{tabular}{lccc}
\toprule
\textbf{Setting} & \textbf{Number of Operands} & \textbf{Range of Input} & \textbf{Range of Answer} \\
\midrule
Train / ID & 3--4 & 100 & 100 \\
OOD        & 4    & 125 & 300 \\
\bottomrule
\end{tabular}
\end{table}






\subsection{GSM8K Task}
\label[app]{sec:appendix_experiment_detail_gsm8k}

\paragraph{Experimental Setup.}

We conduct experiments on the GSM8K benchmark \citep{cobbe2021training}, using the original train/test split, which contains approximately $7.47$K training examples and $1.32$K test examples. We use \texttt{Llama-3.2-2B} \citep{grattafiori2024llama} as the base model. 

\paragraph{Warmup Supervised Fine-Tuning.}

To enable RL training, we first perform supervised fine-tuning on the OpenR1 dataset \citep{face2025open} to strengthen the model's reasoning capabilities. We use the LlamaFactory \citep{zheng2024llamafactory} framework and train for $3{,}064$ steps with batch size $64$ and learning rate $10^{-5}$ using the AdamW optimizer.

\paragraph{RL Training Details.}

We then perform RL training on the GSM8K training split for $300$ steps using the DeepScaleR codebase \citep{luo2025deepscaler}, which is built upon \texttt{verl} \citep{sheng2024hybridflow}.  A response receives reward $1$ if its extracted final answer exactly matches the ground-truth answer, and $0$ otherwise.

\begin{table}[h]
\caption{RL training details for GSM8K experiments.}
\label{tab:gsm8k_rl_details}
\centering
\renewcommand{\arraystretch}{1.15}
\setlength{\tabcolsep}{14pt}
\small

\begin{tabular}{ll}
\toprule
\textbf{Parameter} & \textbf{Value} \\
\midrule
Train batch size & 128 \\
PPO mini-batch size & 64 \\
Max prompt length & 512 \\
Max response length & 4,096 \\
Rollout temperature & 0.6 \\
Rollout top-$k$ & 50 \\
Rollout top-$p$ & 0.95 \\
Responses per prompt & 8 \\
KL coefficient & 0.001 \\
\bottomrule
\end{tabular}
\end{table}


\paragraph{Evaluation Details.}

During evaluation, we sample with temperature $0.6$, top-$p$ $0.95$, and top-$k$ $50$. For MMLU evaluation, we use temperature $0.95$, top-$p$ $0.7$, and top-$k$ $50$, and report the category-averaged MMLU accuracy.

\subsection{DeepScaleR Task}

\paragraph{Experimental Setup.}
We follow the setup of DeepScaleR \citep{luo2025deepscaler} training with their 40k dataset, and evaluating with MATH \citep{hendrycks2021measuring} and AIME2024 and AIME 2025 benchmark. We use DeepSeek-R1-Distill-Qwen-1.5B \citep{deepseekai2025deepseekr1incentivizingreasoningcapability} as the base model.

\paragraph{RL Training Details.}

  We perform RL training using the DeepScaleR codebase \citep{luo2025deepscaler}, which is built on top of \texttt{verl}
  \citep{sheng2024hybridflow}. We follow the training configuration of the first stage (8k context length) of \citet{luo2025deepscaler}. Specifically, for each prompt, we sample 8
  responses with temperature 0.6. The maximum prompt length is 1{,}024 tokens and the maximum response length is 8{,}192 tokens. We use a
  train batch size of 128, a PPO mini-batch size of 64, an actor learning rate of $1\times10^{-6}$, and a KL coefficient of 0.001. We use binary reward.


\paragraph{Evaluation Details.}

During evaluation, we sample with temperature $0.6$, top-$p$ $0.95$, and top-$k$ $50$. 

\subsection{Experimental Details for \cref{sec:distance_insufficient}}
\label{app:detail_distance}
For both models, we generate responses with temperature $0.6$, top-$p$ $0.95$, and a maximum generation length of $100{,}000$ tokens. We evaluate both models on the AIME 2024 and AIME 2025 benchmarks. For each benchmark, each model generates 14 responses per question, yielding 840 responses per model in total. The response length is measured using the Qwen3-8B tokenizer for both models. 

\subsection{Experimental Details for \cref{sec:onlineefficient}}
\label{app:detail_onlineil_condition}

\paragraph{Training Setup.}

  We reproduce the DeepSeek-family setting in Figure~4 (left) of \citet{li2026rethinkingopd}. For completeness, we repeat the experimental details below. We
  use \texttt{DeepSeek-R1-Distill-Qwen-1.5B} as the student model. We consider two expert policies: \texttt{DeepSeek-R1-Distill-Qwen-7B} and \texttt{Skywork-OR1-Math-7B}. We train on
  the \texttt{DAPO-Math-17K} dataset \citep{yu2025dapo} for 270 steps and evaluate performance on \texttt{AIME24} and \texttt{AIME25}.

  We use the OPD training pipeline implemented with \texttt{verl}, using the PPO trainer in the \texttt{verl} \citep{sheng2024hybridflow} codebase
  together with \texttt{vLLM} \citep{kwon2023efficient} for rollout generation. We disable KL regularization
  and use token-mean loss aggregation.

  Both runs share the same optimization hyperparameters and differ only in the expert model. We train with learning
  rate $10^{-6}$ for $270$ steps on $8$ GPUs. The rollout temperature is $1.0$, and we sample $4$ responses per prompt
  during training. We use top-$k$ truncation with $k=16$.
  We use a maximum prompt length of $1024$, a maximum
  response length of $7168$, and a maximum generation length of $31744$ for validation.

\begin{figure}
    \centering
    \includegraphics[width=0.45\linewidth]{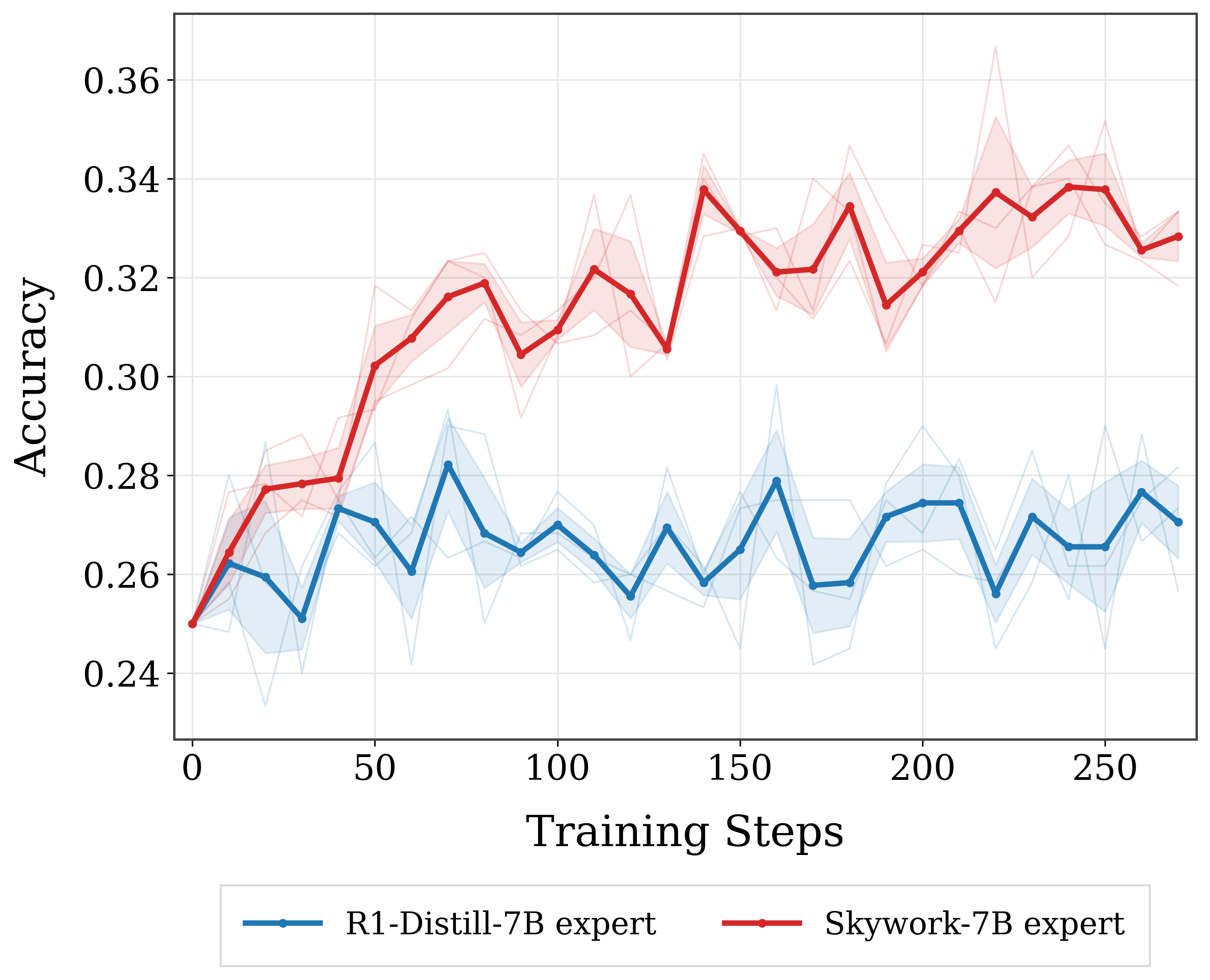}
    \caption{\textbf{The on-policy distillation results in \cref{sec:onlineefficient}.}
Starting from the same base model, DeepSeek-R1-Distill-Qwen-1.5B, we perform on-policy distillation using two 7B experts: R1-Distill-7B and Skywork-7B.
Distillation from Skywork-7B improves average performance on AIME 2024 and 2025 from $25.0\%$ to $32.8\%$, whereas distillation from R1-Distill-7B improves it to $27.1\%$ (averaged over three runs).}
    \label{fig:placeholder}
\end{figure}

\begin{table}[t]
\centering
\caption{Empirical estimation of quantities in \Cref{ass:main}.}
\label{tab:empirical_alignment_estimation}
\begin{tabular}{@{}p{0.55\linewidth}p{0.45\linewidth}@{}}
\toprule
\textbf{Quantity} & \textbf{Empirical proxy} \\
\midrule

$\val(\pi)
=
\E_{\x}
\E_{\y\sim\pi(\cdot\mid\x)}
r(\x,\y)$s
&
$\widehat V(\pi)
=
\frac{1}{M}
\sum_{i=1}^M r(\x_i,\y_i)$
\\[1mm]

$\Jon(\pi)
=
\E_{\x}
\E_{\y\sim\pi(\cdot\mid\x)}
f(\expolicy(\y\mid\x))$
&
$\widehat J_{\mathrm{on}}(\pi)
=
\frac{1}{M}
\sum_{i=1}^M \widehat f(\x_i,\y_i)$
\\[1mm]
policy class $\Pi$
& Collected 162 checkpoints $\widehat \Pi$
\\[1mm]

$\epsilon$-optimal policies w.r.t.\ $\Jon$
&
$\widehat{\Pi}_\epsilon:=
\left\{
\pi \in \widehat{\Pi}:
\widehat{J}_{\mathrm{on}}(\pi)
\ge
\max _{\pi'\in {\widehat \Pi}}\widehat{J}_{\mathrm{on}}(\pi' )-\epsilon
\right\}
$
\\[1mm]

Reference policy $\optimalpolicy$
&
$\hat\pi^\star
=
\arg\max_{\pi\in\widehat \Pi }
\widehat V(\pi)$
\\
\bottomrule
\end{tabular}
\vspace{-3mm}
\end{table}

\paragraph{Estimating the suboptimality bound.}


We perform three distillation runs for each expert and collect checkpoints every 10 steps, obtaining 162 checkpoints in total. We denote this checkpoint set by $\widehat \Pi$.  For each checkpoint policy $\widehat \Pi$, we sample $M=600$ responses on AIME 2024 and AIME 2025.

Each sample consists of a benchmark prompt $x_i$ and a response $y_i\sim\pi(\cdot\mid x_i)$.
The reward value is estimated over these $M$ responses, i.e., 
\[
\widehat V(\pi)
\coloneqq
\frac{1}{M}\sum_{i=1}^M r(x_i,y_i),
\]
where $r(x_i,y_i)$ is the correctness reward.
For the expert score, we use the expert log-probability, i.e. $f(\expolicy(\y\mid \x) ) = 
\log \expolicy(\y\mid \x)$, as is commonly done in on-policy distillation for training stability.
This gives
\[
\widehat J_{\mathrm{on}}(\pi)
\coloneqq
\frac{1}{M}\sum_{i=1}^M \log \expolicy(\y\mid \x).
\]
The resulting estimates are summarized in \Cref{tab:empirical_alignment_estimation}.


We estimate the suboptimality bound in \Cref{thm:main} separately for each expert.
To simplify notation, we omit the expert index in the estimators below.
Let
$
\widehat\pi^\star
\in
\argmax_{\pi\in \widehat \Pi}
\widehat V(\pi)
$
be the empirically best checkpoint policy in terms of reward value.
Using the decomposition in \Cref{eq:f_decomposition}, for a fixed $\alpha>0$ we define the empirical residual
\[
\widehat H_{\alpha}(\pi)
\coloneqq
\widehat J_{\mathrm{on}}(\pi)
-
\alpha \widehat V(\pi)
\]
for some $\alpha$ to be decided later.
The action-independent offset $b(x)$ cancels when comparing two policies, and therefore does not need to be estimated.
We then estimate the residual variation over $\widehat \Pi_\epsilon$:
\[
\widehat\delta_{\alpha}
\coloneqq
\sup_{\pi\in\widehat \Pi_\epsilon}
\left[
\widehat H_{\alpha}(\pi)
-
\widehat H_{\alpha}(\widehat\pi^\star)
\right].
\]
Finally, we report the estimated suboptimality bound
$
\widehat\Delta(\epsilon)
\coloneqq
\min_{\alpha>0}
(\epsilon+\widehat\delta_{\alpha})/\alpha.
$

\end{document}